\documentclass[
    a4paper,
    11pt,
]{article}

\usepackage{microtype}

\usepackage{amssymb}
\usepackage{amsmath}
\usepackage{amsthm}
\usepackage{amsxtra}
\usepackage{thmtools}

\usepackage{multicol}

\usepackage{enumitem}
\setlist{nolistsep}

\usepackage{csquotes}

\usepackage{xspace}
\usepackage{hyphenat}

\usepackage{booktabs} 
\usepackage{verbatim} 
\usepackage{graphicx}
\usepackage{float}
\usepackage{subfig}

\usepackage[plain]{algorithm}
\usepackage{algpseudocode}

\usepackage{url}            
\usepackage{amsfonts}       
\usepackage{nicefrac}       

\usepackage{array}
\newcolumntype{C}{>{$}c<{$}}
\newcolumntype{L}{>{$}l<{$}}
\newcolumntype{R}{>{$}r<{$}}

\usepackage{varwidth}

\usepackage{mathrsfs} 
\usepackage{eufrak} 
\usepackage{pdfpages}

\usepackage[hidelinks,unicode,psdextra]{hyperref}
\usepackage{cleveref} 

\usepackage{tikz}
\usetikzlibrary{arrows,shapes.geometric,matrix}
\usepackage{pgfplots}

\pgfplotsset{compat=1.14}
\algtext*{EndWhile}
\algtext*{EndFor}
\algtext*{EndProcedure}
\algtext*{EndIf}
\algnewcommand\algorithmicinput{\textbf{Input:}}
\algnewcommand\Input{\item[\algorithmicinput]}
\algnewcommand\algorithmicoutput{\textbf{Output:}}
\algnewcommand\Output{\item[\algorithmicoutput]}

\newcommand{\algparbox}[1]{\parbox[t]{\dimexpr\linewidth-\algorithmicindent}{#1\strut}}
\newcommand{\StateN}[1]{\State \algparbox{#1}}

\newlength{\smallfigwidth}
\newlength{\smallfigheight}
\newlength{\smallfigsep}
\newlength{\legendheight}
\allowdisplaybreaks[1]

\declaretheorem{theorem}
\declaretheorem[numberlike=theorem]{lemma}

\declaretheorem[numberlike=theorem]{proposition}

\newcommand{\Exp}[1]{\mathbf{E}\left[ #1 \right]}
\newcommand{\Prob}[1]{\mathbf{Pr}\left( #1 \right)}

\DeclareMathOperator*{\argmin}{arg\,min}
\DeclareMathOperator*{\argmax}{arg\,max}

\newcommand{\norm}[1]{\left|\left| #1 \right|\right|_2}
\newcommand{\abs}[1]{\lvert #1\rvert}

\newcommand{\cost}{\textit{cost}}

\renewcommand{\lg}{\log}
\newcommand{\LB}{\textit{LB}}
\newcommand{\poly}{\operatorname{poly}}

\newcommand{\score}{\textit{score}}

\DeclareRobustCommand{\ALG}{%
	\ifmmode
		\operatorname{ONL}
	\else
		\text{ONL}\xspace
	\fi
}
\DeclareRobustCommand{\OFF}{%
	\ifmmode
		\operatorname{OPT}
	\else
		\text{OPT}\xspace
	\fi
}
\DeclareRobustCommand{\APPROXALGO}{%
	\ifmmode
		\operatorname{APPROX}
	\else
		\text{APPROX}\xspace
	\fi
}

\newcommand{\sofa}{\textsf{sofa}\xspace}
\newcommand{\sofaauto}{\textsf{sofa-auto}\xspace}
\newcommand{\staticsofa}{\textsf{static sofa}\xspace}
\newcommand{\AlgoName}{\sofa}
\newcommand{\AlgoToy}{\textsf{Greedy-clustering}\xspace}
\newcommand{\MG}{\textsf{MG}\xspace}

\newcommand{\basso}{\textsf{basso}\xspace}

\newcommand{\asso}{\textsf{asso}\xspace}
\newcommand{\zha}{\textsf{RSzhaEtAl}\xspace}
\newcommand{\dhillon}{\textsf{RSdhillon}\xspace}

\newcommand{\News}{\textsf{20News}\xspace}
\newcommand{\Movielens}{\textsf{Movie}\xspace}
\newcommand{\Book}{\textsf{Book}\xspace}
\newcommand{\Wikipedia}{\textsf{Wiki}\xspace}
\newcommand{\Reuters}{\textsf{Reuters}\xspace}
\newcommand{\Flickr}{\textsf{Flickr}\xspace}

\AtBeginDocument{\renewcommand{\setminus}{\mathbin{\backslash}}}

\newcommand{\codeurl}{https://cs.uef.fi/~pauli/bmf/sofa/}

\title{Biclustering and Boolean Matrix Factorization in Data Streams\footnote{This
	technical report is the slightly extended version of a
	paper~\cite{neumann20biclustering} which appeared at VLDB'20.}}
\author{Stefan Neumann\footnote{KTH
		Royal Institute of Technology, Stockholm, Sweden. This work was done
		while SN was at the University of Vienna and during visits at Brown
		University and University of Easter Finland.  neum@kth.se.} \and 
	Pauli Miettinen\footnote{University of Eastern Finland, Kuopio, Finland.
	pauli.miettinen@uef.fi.}}
\date{}

\begin{document}

\maketitle

\graphicspath{{figs/}}

\begin{abstract}
	\noindent
	We study the clustering of bipartite graphs and Boolean matrix factorization
	in data streams.
	We consider a streaming setting in which the vertices from the \emph{left}
	side of the graph arrive one by one together with all of their incident
	edges.  We provide an algorithm that, after one pass over the stream,
	recovers the set of clusters on the \emph{right} side of the graph using
	sublinear space; to the best of our knowledge, this is the first algorithm
	with this property.  We also show that after a second pass over the stream,
	the left clusters of the bipartite graph can be recovered and we show how to
	extend our algorithm to solve the Boolean matrix factorization problem (by
	exploiting the correspondence of Boolean matrices and bipartite graphs).  We
	evaluate an implementation of the algorithm on synthetic data and on
	real-world data. On real-world datasets the algorithm is orders of
	magnitudes faster than a static baseline algorithm while providing quality
	results within a factor~2 of the baseline algorithm.  Our algorithm scales
	linearly in the number of edges in the graph.  Finally, we analyze the
	algorithm theoretically and provide sufficient conditions under which the
	algorithm recovers a set of planted clusters under a standard random graph
	model.
\end{abstract}

\section{Introduction}
\label{sec:bmf:introduction}

Bipartite graphs appear in many areas in which interactions of objects from two
different domains are observed. Hence, finding interesting clusters (also
called communities) in bipartite graphs is a fundamental and
well-researched problem with many applications; this problem is often
called \emph{biclustering}.  For example, in social networks the two domains
could be users and hashtags and an interaction corresponds to a user using a
certain hashtag; finding clusters in such a graph corresponds to finding groups
of hashtags used by the same users and groups of users using the same
hashtags~\cite{wang19sgp}. Biclustering has many applications across many
domains such as computational
biology~\cite{madeira04biclustering,eren13comparative}, text
mining~\cite{dhillon01coclustering} and finance~\cite{huang15biclustering}.

Many real-world bipartite graphs have three natural properties. First, the
numbers of vertices on both sides of the graphs are very large, while at the
same time their density is extremely low, i.e., the graphs are very sparse. For
example, consider a bipartite graph consisting of users on the left side of the
graph and movies on the right side of the graph, where an edge indicates that a
user rated a movie. Such graphs often consist of millions of users and movies,
but the average degree is constant.  Second, the degrees on one side of the
graph are usually bounded by a small constant, while on the other side of the
graph a few vertices have extremely large degrees.  Continuing the example from
above, note that users usually do not rate more than 1000~movies, but a small
number of popular movies is rated by millions of users.  The third property is
that the clusters on the high-degree side of the graph are usually
relatively small.  Again continuing the above example, those groups of movies
which are watched by the same users typically do not consist of more than
50~movies.

Furthermore, for many real-world bipartite graphs, it is natural to assume that
the left-side vertices appear in a data stream, for example,
in Natural Language Processing~\cite{goyal09streaming}, market basket analysis,
network traffic analysis and stock price analysis~\cite{calders14mining}.
For instance, in market basket analysis the left-side vertices correspond to
transactions in a supermarket and incident edges indicate which products were
bought.  To efficiently find interesting clusters in such datasets, we need to
develop streaming algorithms for the biclustering problem. Another motivation to
study the streaming setting is that current static algorithms do not scale to
the previously mentioned real-world graphs with millions of vertices, because
these methods have prohibitively high memory consumptions and running times.
Streaming algorithms could mitigate this issue due to improved memory efficiency
and speed.

\subsection*{Our Contributions}
We address this question and provide the first streaming
algorithms for the biclustering problem.  In particular, we study a streaming
setting in which the vertices from the left side of the bipartite graph arrive
one after another, together with all of their incident edges.  Then after a
single pass over the stream, the algorithm must output the set of right-side
clusters of the graph.  The algorithm is then allowed a second pass over the
stream in order to output the left-side clusters. See
Section~\ref{sec:bmf:preliminaries:biclustering} for the formal definition of the
problem.

To obtain our algorithms, we heavily exploit the previously mentioned properties
of real-world bipartite graphs. Formally, we assume that there exists a
number~$s$ such that the degree of all left-side vertices and the size of all
right-side clusters is at most~$s$. This implies that, in total, the graph
contains $O(ms)$ edges, where $m$ is the number of vertices on the left side of
the graph.

We introduce the \sofa algorithm which returns the right-side clusters of the graph
after a single pass over the stream and using sublinear memory. To the best our
of knowledge, \sofa is the first algorithm with this property. The running time of
\sofa is $O(ms \cdot k\lg m)$, where $k$ is the number of clusters to be recovered;
note that this running time is within a $O(k \lg m)$ factor of the size of the
graph.  During its running time, \sofa uses $O(ks \lg m)$ space; observe that
this space usage sublinear in the size in the size of the graph as long as
$s=o(m/\lg m)$ which is realistic in practice (as we argued before). Furthermore,
we show that the left-side clusters of the graph can be computed using a second
pass over the stream and using space $O(m)$, which is optimal since we have to
output a cluster assignment for each of the $m$ left-side vertices of the graph.

We also provide theoretical guarantees for a version of \sofa. We show that
under a standard random graph model, a version of \sofa returns a set of planted
ground-truth clusters with information-theoretically optimal memory usage; see
Theorem~\ref{thm:bmf:algo-toy} for details. We
also provide similar yet weaker guarantees for the practical version of \sofa.

Next, we show how \sofa can be extended to solve the Boolean matrix
factorization problem, which is popular in the data mining and machine learning
communities.
We obtain similar guarantees on space and run-time as above.
Unfortunately, we cannot provide any quality guarantees here, because the lower
bounds from~\cite{chandran16parameterized} rule out obtaining non-trivial
approximation ratios for practical BMF algorithms (see
Section~\ref{sec:bmf:related} for details). Thus, \sofa is a
heuristic for BMF, but our experiments show that it works well in practice.

We evaluate \sofa on synthetic as well as on real-world datasets. On
synthetically generated random graphs, our experiments show that \sofa returns
clusters, that are close to the planted ground-truth clusters and that its
running time scales linearly in the number of edges in the graph.
On real-world datasets, our experiments show that \sofa is orders of magnitudes
faster and more memory-efficient than a static baseline algorithm, while at the
same time achieving objective function values within factor~2 of the baseline.
In concrete terms, \sofa can process a graph with millions of vertices, for
which the static baseline algorithm runs out of memory, using only 500~MB of RAM
and, further, \sofa can process a graph with hundreds of thousands of edges
within less than three hours, while the baseline algorithm requires several days
to finish.

\textbf{Outline of the Paper.}
The paper is arranged as follows. In Section~\ref{sec:bmf:preliminaries} we formally
define the problems we study. Then in Section~\ref{sec:bmf:algorithm-right} we
introduce \sofa, which performs a single pass over the left side of a bipartite
graph and then returns the right-side clusters. We show how the left-side clusters
can be recovered during a second pass over the stream in
Section~\ref{sec:bmf:algorithm-left}. In Section~\ref{sec:bmf:implementation}, we discuss
certain adjustments of the algorithms that we made during the implementation and
then we evaluate \sofa experimentally in Section~\ref{sec:bmf:experiments}.
Section~\ref{sec:bmf:theory} contains our theoretical
analysis.
We discuss related work in Section~\ref{sec:bmf:related} and conclude in
Section~\ref{sec:bmf:conclusion}.

\section{Preliminaries}
\label{sec:bmf:preliminaries}
\label{sec:bmf:biclustering}
\label{sec:bmf:bmf}
\label{sec:bmf:relationship}
\label{sec:bmf:heavy-hitters}

In this section, we formally introduce the problems we study, we discuss their
relationship and we introduce an important subroutine of our algorithms.

\subsection{Biclustering in Random Graphs}
\label{sec:bmf:preliminaries:biclustering}
We study biclustering of random bipartite graphs. Let $G = (U \cup V, E)$ be a
bipartite graph, where $U$ is the set of vertices on the left side of the graph
and $V$ is the set of vertices on the right side of the graph. We assume that
$U$ is partitioned into subsets $U_1,\dots,U_k$ for $k>1$ and $V_1,\dots,V_k$
are subsets of $V$ (it is not necessary that the $V_j$ are mutually disjoint or
that their union is the set $V$).

Now let $p,q \in [0,1]$ be probabilities with $p > q$. In our random model,
vertices $u \in U_i$ have edges to vertices $v\in V_i$ with ``large'' probability and
to vertices in $v\in V_{j}$ with $i\neq j$ with ``low'' probability. More
concretely, we assume that
\begin{align}
\label{eq:bmf:random-edges}
	\Prob{ (u,v) \in E } =
	\begin{cases}
		p, & \text{if } u \in U_i, v \in V_i, \\
		q, & \text{if } u \in U_i, v \in V_{j}, i \neq j.
	\end{cases}
\end{align}

Now the computational problem is as follows. We assume that our algorithms
obtain as input a graph $G$ generated from the random model above and the
parameters $k$, $p$ and $q$ (but have no knowledge about the sets
$U_i$ and $V_j$). The task is to recover the clusters $U_i$ and
$V_j$ from $G$; that is, the algorithm must output clusters
$\tilde{U_1},\dots,\tilde{U_k} \subseteq U$ and $\tilde{V_1},\dots,\tilde{V_k}
\subseteq V$, such that $\{ \tilde{U_1},\dots,\tilde{U_k} \} = \{ U_1,\dots,U_k
\}$ and $\{ \tilde{V_1},\dots,\tilde{V_k} \} = \{ V_1,\dots,V_k \}$.

We decided to study the above random graph model for two reasons. First, the
model has been widely studied theoretically, e.g., in machine
learning~\cite{neumann18bipartite,zhou19analysis} and in
mathematics~\cite{abbe16community,zhou18optimal},  and similar models have been
used to derive practical algorithms~\cite{ravanbakhsh16boolean,rukat17bayesian}.
Second, when dropping the random graph assumption and
assuming worst-case inputs, biclustering problems are
NP-hard~\cite{orlin1977contentment} and require prohibitively high running
times~\cite{chandran16parameterized}.

In the streaming setting, the algorithm's input is a stream of the left-side
vertices $u\in U$, where each vertex arrives together with all of its incident
edges. We further assume that for some parameter $s$, each $u\in U$ has at most
$s$ incident edges and that $|V_i|\leq s$ for all $i$. Note that the stream only
contains left-side vertices $u\in U$ and does \emph{not} contain the vertices
$v\in V$. After single pass over the stream, the algorithm must return the
right-side clusters $\tilde{V_i}$. Then, the algorithm is allowed a second pass
over the stream to output the left-side clusters $\tilde{U_i}$.

Next, we state our theoretical guarantees.  We prove that after a single pass
over the left-side vertices of a bipartite graph, the planted right-side
clusters can be recovered if some conditions hold.  We write
$A\triangle B = (A \setminus B) \cup (B \setminus A)$ to denote the symmetric
difference.
\begin{theorem}
\label{thm:bmf:algo-toy}
	Let $G = (U\cup V,E)$ be a random bipartite graph with planted clusters
	$U_1,\dots,U_k$ and $V_1,\dots,V_k$ as above.
	Let $p \in [1/2, 0.99]$ and $s=\max_i \abs{V_i}$.
	There exist constants $K_1,K_2,K_3,K_4$ such that if

 	\smallskip
 	\begin{varwidth}[t]{.48\columnwidth}
 		\begin{itemize}
 			\item $q \leq K_1 p s / n$,
 			\item $\abs{V_i} \geq K_3 \lg n$ for all $i$,
 		\end{itemize}
 	\end{varwidth}
 	\hspace{4em}
 	\begin{varwidth}[t]{.52\columnwidth}
 		\begin{itemize}
 			\item $\abs{U_i} \geq K_2 \lg n$ for all $i$,
 			\item $\abs{V_i \triangle V_{i'}} \geq K_4 s$ for $i \neq i'$,
 		\end{itemize}
	\end{varwidth}\\
 	\smallskip
	then there exists an algorithm which
	returns clusters
	$\tilde{V_1},\dots,\tilde{V_k}$ such that
	with high probability
	$\{\tilde{V_1},\dots,\tilde{V_k}\} = \{V_1,\dots,V_k\}$.
	The algorithm uses $O(ks)$ space and has a running time of $O(mks)$.
\end{theorem}

Let us briefly discuss this result and for simplicity assume that the $V_i$ are
disjoint and have size $\abs{V_i} = s = \Omega(\lg n)$. Then the bounds for $p$ and $q$ essentially
require that $p > 1/2$, $q \approx p s / n$ and $\abs{U_i}=\Omega(\lg n)$.  While
this is much weaker than bounds derived for static algorithms for this type of
random graph model (e.g., \cite{neumann18bipartite, zhou19analysis}), the static
algorithms do not use sublinear space. Furthermore, the bounds on $p$ and $q$
are almost optimal when one wants to ensure that a greedy clustering of the
left-side vertices succeeds.\footnote{
	Roughly speaking, the condition on $q$ ensures that the left-side vertices have
	more ``signal edges'' than ``noise edges''. More concretely, in our setting
	with small right-side clusters $V_i$ of size $|V_i| \approx s \ll n$, we
	have that $n - s \approx n$. Thus, in expectation every vertex $u\in U_i$
	has $ps$ ``signal-edges'' to vertices from its corresponding right-side
	cluster $V_i$ and $q (n-s) \approx qn$ ``noise-edges'' to vertices in
	$V\setminus V_i$. Now, if $q \gg p s / n$, then $u$ has $qn \gg p s / n
	\cdot n = p s$ ``noise-edges'' and, hence, more ``noise edges'' than
	``signal-edges''. In such a case, the Hamming distances of vertices from the
	same cluster $U_i$ are essentially identical to the Hamming distances of
	vertices from different clusters $U_i$ and $U_j$, $i\neq j$.  Therefore,
	clustering the vertices in $U$ based on their Hamming distance cannot
	succeed anymore and, hence, the analysis of our algorithm is tight w.r.t.\
	the choice of $q$.
}

We also show that \emph{any} algorithm recovering the planted right-side
clusters must use space $\Omega(ks)$. Thus, the space usage of the algorithm
from the theorem is optimal.  We prove the theorem and the proposition in
Section~\ref{sec:bmf:theory}.
\begin{proposition}
\label{prop:bmf:space}
	Any algorithm solving the above biclustering problem
	requires at least $\Omega(ks)$ space.
\end{proposition}

\subsection{Boolean Matrix Factorization (BMF)}
\label{sec:bmf:preliminaries:bmf}
In the Boolean Matrix Factorization (BMF) problem, the input is a matrix
$B\in\{0,1\}^{m\times n}$ and the task is to find factor matrices
$L\in\{0,1\}^{m\times k}$ and $R\in\{0,1\}^{k\times n}$ such that
$\norm{B - L \circ R}$ is minimized. Here, $\circ$ denotes matrix
multiplication under the Boolean algebra, i.e., for all $i=1,\dots,m$ and
$j=1,\dots,n$,
\[(L \circ R)_{ij} = \bigvee_{r=1}^k (L_{ir} \land R_{rj})\; .\]

In the streaming setting, the algorithm's input is a stream
consisting of the rows $B_i$ of $B$, where we assume that each row $B_i$ has at
most $s$ non-zero entries. After a single pass over the stream, the
algorithm must output the right factor matrix $R$. Then, the algorithm is
allowed a second pass over the stream to compute the left factor matrix $L$.

While the biclustering problem and the BMF problem might appear quite different
at first glance, they are tightly connected. Indeed, there is a one-to-one
correspondence between bipartite graphs $G = (U\cup V, E)$ with
$U=\{u_1,\dots,u_m\}$ and $V=\{v_1,\dots,v_n\}$ and Boolean matrices
$B\in\{0,1\}^{m\times n}$: The rows of $B$ correspond to the vertices $u_i\in U$ and
the columns of $B$ correspond to the vertices $v_j\in V$; now one sets $B_{ij} = 1$
iff $(u_i,v_j)\in E$. This is yields a bijective mapping between bipartite
graphs and Boolean matrices; $B$ is often called the \emph{biadjacency matrix}
of $G$.

Furthermore, there is a correspondence of clusterings
$U_1,\dots,U_k\subseteq U$ and $V_1,\dots,V_k\subseteq V$ and the factor
matrices $L$ and $R$: The clusters $U_i$ correspond to the columns
of $L$ and the clusters $V_j$ correspond to the rows of $R$. More precisely,
consider the $r$'th column of $L$ and set it to the indicator vector of $U_k$,
i.e., we set $L_{ir}=1$ iff $u_i \in U_r$.  Simililary, we set $R_{rj}=1$ iff
$v_j\in V_r$.

There are two main differences between the problems. First, while in
biclustering we try to recover a set of planted ground-truth clusters, in BMF we
try to optimize an objective function. However, when $p>1/2>q$, a ``good''
biclustering solution will also provide a good BMF solution and vice versa.
Second, in biclustering each vertex $u\in U$ belongs to exactly one cluster
$U_i$ (since the $U_i$ partition $U$). This would correspond to the constraint
in BMF that each column of the factor matrix $L$ must contain exactly one
non-zero entry. However, in BMF we do not make this assumption and allow each
column of $L$ to contain arbitrarily many non-zero entries. Thus, in BMF the
vertices $u\in U$ are allowed to be member of multiple clusters
$U_{i_1},\dots,U_{i_t}$ (and the clusters $U_i$ do not have to be mutually
disjoint). To address these differences, in Section~\ref{sec:bmf:algorithm-left} we
use different algorithms for computing the left-side clusters $U_i$ for
biclustering and for BMF.

\subsection{Mergeable Heavy Hitters Data Structures}
\label{sec:bmf:preliminaries:heavy-hitters}
Next, we recap mergeable heavy hitters data structures, which we will use
as subroutines in our algorithms.

Let $X = (e_1,\dots,e_N)$ be a stream of elements from a discrete domain $A$.
The \emph{frequency} $f_a$ of an element $a \in A$ is its number of occurrences
in the stream, i.e., $f_a = \abs{\{i : e_i = a \}}$. In the \emph{heavy hitters}
problem the task is to output all
elements with $f_a \geq \varepsilon N$ and none with $f_a < \varepsilon N/2$
after a single pass over the stream for $\varepsilon>0$.

Misra and Gries~\cite{misra82finding} provided a data structure which solves the
heavy hitters problem using $O(1/\varepsilon)$ space. In fact, their data
structure can approximate the frequency of each element $a \in A$ with additive
error at most $\varepsilon N/2$. For the rest of the paper, we will denote
Misra--Griess data structures by \MG.

Agarwal et al.~\cite{agarwal13mergeable} showed that Misra--Gries data
structures are \emph{mergeable}: Let $\MG_1$ and $\MG_2$ be two Misra--Gries
data structures which were constructed on two different streams $X_1$ and $X_2$.
Then there exists a merge algorithm which on input $\MG_1$ and $\MG_2$
constructs a new data structure, that satisfies the same guarantees as a
Misra--Gries data structure which was built on the concatenated stream $X_1 \cup
X_2$. We write $\MG_1 \cup \MG_2$ to denote such a merged data structure.

\emph{Remark.} While we use the mergeable version of the Misra--Gries data
structure, we could as well other mergeable heavy hitters data structures such
as the count-min sketch~\cite{cormode05improved}.  See~\cite{agarwal13mergeable}
for more details on mergeable data structures.

\section{First Pass: Recover Right Clusters}
\label{sec:bmf:algorithm-right}

We describe two algorithms for computing the right clusters~$\tilde{V_j}$. As described
in Section~\ref{sec:bmf:biclustering}, we assume that the algorithms obtain as input
a stream $U = (u_1,\dots,u_m)$ consisting of vertices from the left side of the
graph, where each $u_i$ arrives together with all of its at most $s$ edges to
vertices on the right side of the graph.  After a single pass over $U$, the
algorithm must return clusters $\tilde{V_1},\dots,\tilde{V_k}$ on the right side
of the graph.

It will be convenient to identify the vertices $u \in U$ with
bit-vectors $x_u \in \{0,1\}^n$, where we set $x_u(j) = 1$ iff $(u,v_j)\in E$,
i.e., $x_u(j) = 1$ iff vertex $u$ is a neighbor of $v_j \in V$.  For two
vertices $u, u' \in U$, we let $d(x_u,x_{u'})= |\{ j : x_u(j) \neq x_{u'}(j) \}|$
denote the Hamming distance
of $x_u$ and $x_{u'}$, i.e., 
$d(x_u,x_{u'})$ measures the number of vertices in
$V$ which are incident upon $u$ or $u'$ but not both.

We will first describe a simplified greedy algorithm to
highlight our main ideas; this is the algorithm mentioned in
Theorem~\ref{thm:bmf:algo-toy}. Then we provide a second, more practical, algorithm in
Section~\ref{sec:bmf:algorithm-right:importance}; we implement and evaluate this
algorithm in Sections~\ref{sec:bmf:implementation} and~\ref{sec:bmf:experiments}.

\subsection{Warm Up: Greedy Biclustering}
\label{sec:bmf:algorithm-right:greedy}

We start by discussing a simplified greedy algorithm to explain the main idea of
our approach.  This greedy algorithm has the guarantees stated in
Theorem~\ref{thm:bmf:algo-toy}.

Before describing the algorithm, let us first make two observations about the
properties of the random graph model in
Section~\ref{sec:bmf:preliminaries:biclustering}:
(1)~Suppose we know a planted left-side cluster $U_i$ and we want to recover its
corresponding right-side cluster $V_i$. Then observe that by
Equation~\ref{eq:bmf:random-edges} every vertex $v \in V_i$ has $p |U_i|$ neighbors
in $U_i$ and every vertex $v \not\in V_i$ has $q |U_i|$ neighbors in $U_i$.
Thus, if $U_i$ is large enough and $p$ is sufficiently larger than $q$, we can
find a threshold $\theta$ such that  with high probability all $v\in V_i$ have
more than $\theta |U_i|$ neighbors in $U_i$ and all $v\not\in V_i$ have less
than $\theta |U_i|$ neighbors in $U_i$. Hence, recovering the cluster $V_i$
essentially boils down to identifying those vertices in $V_i$ which are
frequently neighbors of vertices in $U_i$. In other words, we want to find the
heavy hitters among the neighbors of vertices in $U_i$.
(2)~The second insight is that when processing the stream, the vertices
$u,u'\in U_i$ from the same cluster will have similar neighborhoods in $V$ and,
hence, $d(x_u,x_{u'})$ is small. More concretely, assume that
$d(x_u,x_{u'})<\alpha$ for some suitable parameter $\alpha$. On the other hand,
if $u\in U_i$ and $u''\in U_j$ with $i \neq j$, their neighborhoods will be
quite different and $d(x_u,x_{u''}) > \alpha$ is large. Thus, a greedy
clustering of the vertices $u\in U$ based on the distances of their
corresponding vectors $x_u$ suffices to recover the $U_i$. In
Section~\ref{sec:bmf:theory}, we show how $\theta$ and $\alpha$ can be picked under
the conditions from Theorem~\ref{thm:bmf:algo-toy}.

Roughly speaking, the algorithm works as follows. It assumes that it obtains
parameters $\theta$ and $\alpha$ with the above properties as input.  Now
the algorithm greedily forms clusters of all left-side vertices which have
distance at most $\alpha$; this corresponds to Observation~(2) above.  To save
memory, the algorithm only stores \emph{a single vertex} for each cluster.
Furthermore, for each cluster consisting of left-side vertices, the algorithm
keeps track how many of its edges are incident upon each right-side vertex
$v\in V$.  Since we do not have enough memory to store a counter for each
vertex $v\in V$, the algorithm uses the
mergeable heavy hitters data structure from
Section~\ref{sec:bmf:preliminaries:heavy-hitters} to
approximately keep track of how many times each right-side vertex appeared; this
corresponds to Observation~(1) above.

Now we describe the algorithm more formally and present its pseudocode in
Algorithm~\ref{algo:bmf:algo-toy}.  The algorithm obtains as input $U$, a distance
parameter~$\alpha$ and a rounding threshold~$\theta$. It maintains a
set of \emph{centers} $C$ which is initially empty. For each center $c\in C$,
the algorithm stores a heavy hitters data structure $\MG(c)$ with $O(s)$
counters
and a counter $n_c$ denoting
how many vertices have been assigned to $c$.

Now the algorithm processes the vertices $u \in U$ as follows.  First, it checks
whether $x_u$ has Hamming distance more than $\alpha$ from all centers $c \in
C$. If this is the case, the algorithm opens $u$ as a new center. That is, it
sets $C \leftarrow C \cup \{u\}$ and sets $n_u \leftarrow 1$. Else, there exists
a center $c(u) \in C$ with $d(x_u, x_{c(u)}) \leq \alpha$ and the algorithm
\emph{assigns} $u$ to $c(u)$. When assigning $u$ to $c(u)$,
the algorithm first creates a heavy hitters data structure $\MG(u)$ containing
all $j$ such that $(u,v_j) \in E$ (note that the algorithm has access to this
information since $u$ arrives together with all of its incident edges). Then it
merges $\MG(c(u))$ and $\MG(u)$ and updates $\MG(c(u))$ to this merged heavy
hitters data structure.  Furthermore, the algorithm increases the counter
$n_{c(u)}$ by $1$. Then it proceeds with the next point from the stream.

When the algorithm finished processing the stream, it performs a postprocessing
step. It iterates over all centers $c\in C$ and sets $\tilde V_c$ to all
vertices $v_j \in V$ such that the counter of $j$ in $\MG(C)$ is at least
$\theta n_c$, where $\theta$ is the rounding threshold from the input and $n_c$
is the number of vertices that were assigned to $c$. Then the algorithm outputs
the clusters $\tilde{V_c}$ as its solution.

\begin{algorithm}[tb]
  \begin{minipage}{\columnwidth}
    \small
  \caption{\AlgoToy($U$, $\alpha$,
		  $\theta$)}\label{algo:bmf:algo-toy}\label{algo:bmf:greedy}
\begin{algorithmic}[1]
	\StateN{$C \leftarrow \emptyset$}
	\For{$u \leftarrow$ next vertex from stream}
		\StateN{$d \leftarrow \min_{c\in C} d(x_u,x_c)$}
		\If{$d > \alpha$} \Comment{open $u$ as center}
			\State{$C \leftarrow C \cup \{u\}$}
			\State{$n_u \leftarrow 1$}
		\Else \Comment{Assign $u$ to its closest center $c(u)$}
			\StateN{$c(u) \leftarrow \argmin_{c\in C} d(x_u,x_c)$}
			\StateN{$\MG(c(u)) \leftarrow \MG(c(u)) \cup \MG(u)$}
			\State{$n_{c(u)} \leftarrow n_{c(u)} + 1$}
		\EndIf
	\EndFor

	\ForAll{$c \in C$} \Comment{Postprocessing}
		\StateN{$\tilde{V_c} \leftarrow \{v_j \in V : $ the counter of $j$ in
			$\MG(c)$ is at least $\theta n_c \}$}
	\EndFor
\end{algorithmic}
\end{minipage}
\end{algorithm}

\emph{Remark.} Note that Algorithm~\ref{algo:bmf:greedy} only delivers good results
when the parameters $\alpha$ and $\theta$ provide exactly those guarantees which
we discussed at the beginning of the subsection. In Section~\ref{sec:bmf:theory} we
show how $\alpha$ and $\theta$ can be set when the parameters $p$, $q$ and $k$
are known for random graph models as introduced in
Section~\ref{sec:bmf:preliminaries:biclustering}; under this assumption we show that the
algorithm indeed returns the planted clusters $V_1,\dots,V_k$ after a single
pass over the stream and using essentially optimal space. However, in practice
it is unrealistic that one has knowledge about these parameters.
Especially setting the parameter $\alpha$ seems troublesome; for example, when
setting $\alpha$ incorrectly, one cannot even guarantee to obtain $k$ clusters
in total. We show how to resolve this issue in the next subsection.

\subsection{Biclustering Using Importance Sampling}
\label{sec:bmf:algorithm-right:importance}
We introduce the \sofa algorithm which constitutes our main contribution;
\sofa is short for \emph{Streaming bOolean FactorizAtion}. \sofa performs
a single pass over the vertices $u\in U$ and afterwards returns clusters
$\tilde{V_1},\dots,\tilde{V_k}$. One can view \sofa as the more practical
version of Algorithm~\ref{algo:bmf:greedy}, since it
it does not require the parameter $\alpha$ which is not available in practice.
In a nutshell, we will replace the greedy
clustering from Algorithm~\ref{algo:bmf:algo-toy} by the streaming $k$-Medians
algorithm from Braverman et al.~\cite{braverman11streaming} which is based on
importance sampling. The pseudocode of \sofa with all details is presented in
Algorithm~\ref{algo:bmf:algo-name}.

Roughly speaking, \sofa works as follows. \sofa maintains a set of
\emph{centers} $C$ which is initially empty; we impose that $C$ is never allowed
to contain more than $c_{\max}$ vertices, where $c_{\max}$ is a user-defined
parameter. As before, for each center $c \in C$, the algorithm maintains a heavy
hitters data structure $\MG(c)$.  When \sofa processes the vertices from the
stream and a new vertex $u$ arrives, \sofa computes the distance
$d=d(x_u,x_{c(u)})$ from $u$ to the closest center $c(u)$ in $C$.  It then opens
$u$ as new center with probability proportional to $d$; if $u$ is not opened as
a center, \sofa assigns $u$ to $c(u)$. Thus, if $u$ is ``close'' to $c(u)$ then
$u$ is unlikely to become a new center and more likely to be assigned to $c(u)$;
on the other hand, if $u$ is ``far away'' from $c(u)$ (and, hence, all centers),
then $u$ is likely to become a new center. As before, when a vertex $u$ is
assigned to $c(u)$, the indices of all neighbors of $u$ are added to
$\MG(c(u))$. Next, suppose that after opening a new center, the set $C$ contains
$c_{\max}$ centers. Then \sofa restarts on the stream which only consists of the
$c_{\max}$ centers in $C$ and all unprocessed vertices of the stream. When \sofa
restarts on the centers of $C$ and one of the previous centers $c_i$ is assigned
to another previous center $c_j$, then \sofa merges their corresponding heavy
hitters data structures $\MG(c_i)$ and $\MG(c_j)$ as described in
Section~\ref{sec:bmf:heavy-hitters}.  Finally, after processing all vertices from
the stream and obtaining a set of centers $C$ together with their heavy hitters
data structures, we run a postprocessing step. At this point $C$ can
contain more than $k$ centers (but at most $c_{\max}$). We run a static
$k$-Medians algorithm on the vectors $x_c$ for $c \in C$ to obtain a clustering
of $C$ into subsets $C_1,\dots,C_k$.  For each $C_i$, we merge the heavy hitters
data structures of the centers in $C_i$ and denote this merged data structure as
$\MG_i$.  As before, we set $\tilde{V_i}$ to all vertices $v_j \in V$
which have a counter of value at least $\theta |C_i|$ in $\MG_i$.

We now elaborate on the details of \sofa.  At the beginning, \sofa initializes a
lower bound $\LB$ on the $k$-Medians clustering cost of the points $x_u$ in the
stream to~$1$. It also maintains an approximation of the current cost of the
clustering which we denote $\cost$ and initialize to $0$. After that, \sofa
starts processing the vertices from the stream.  We
maintain a set of centers $C$ for which we ensure that $|C|<c_{\max}$ at all
times. For each center we store a heavy hitters data structure from
Section~\ref{sec:bmf:preliminaries:heavy-hitters} with $O(s)$ counters.

When starting to process the vertices from the stream, \sofa computes a weight
$f \leftarrow LB/(k(1+\lg n))$.  As long as there are unread vertices in the
stream, $|C| < c_{\max}$ and $\cost < 2\LB$, \sofa proceeds as follows. It reads
the next vertex $u$ from the stream and sets $d$ to the distance
$d(x_u,x_{c(u)})$ of $u$ to its closest center $c(u)$. Now it opens $u$ as a new
center with probability $\min\{w(u) \cdot d/f, 1\}$, where $w(u)$ is the weight
of $u$. \sofa maintains as invariant that if $u$ was a previously unprocessed
vertex from the stream, then $w(u) = 1$, and, if $u$ was a center before, then
$w(u)$ is the number of vertices which were previously assigned to $u$. If $u$
is opened as a new center, we set $C \leftarrow C \cup \{u\}$.  If $u$ is
assigned to its closest center $c(u)$, then we increase $\cost$ by $w(u) \cdot
d$, increase the weight of $c(u)$ by $w(u)$ and set $\MG(c(u))$ to the merged
heavy hitters data structures of $\MG(c(u))$ and $\MG(u)$.

If at some point $|C| = c_{\max}$ or $\cost >2\LB$, then \sofa doubles $\LB$.
Furthermore, \sofa restarts on the stream which consists of
the $c_{\max}$ vertices of $C$ and all unprocessed vertices from $U$ (in this
order).  Note that the vertices $c \in C$ still have their previously assigned
weights $w(c)$, whereas the vertices in the unprocessed part of $U$ all have
weight $1$.

After \sofa finished processing all vertices from the stream, we perform
a postprocessing step. We start by running a static $O(1)$-approximate
$k$-Medians algorithm on the points $x_c$ for $c \in C$ which uses only
$O(|C| \cdot s)$ space and which runs in time $\poly(|C| \cdot s)$;
this can be done, for example, using the local search algorithm by Arya et
al.~\cite{arya04local}.  This provides us with a clustering of $C$
into disjoint subsets $C_1,\dots,C_k$.  Now for each $i=1,\dots,k$, we set
$\MG_i$ to the merged heavy hitters data structure of all vertices in $C_i$ and
$|C_i|$ to the sum of the weights of all vertices in $C_i$.  Finally, we set
$\tilde{V_i}$ to all vertices $v_j \in V$ such that the counter of $j$ in
$\MG_i$ is at least $\theta |C_i|$.

\begin{algorithm}[tb]
\begin{minipage}{\columnwidth}
  \caption{\AlgoName($U$, $k$, $c_{\max}$,
		  $\theta$)}\label{algo:bmf:importance}\label{algo:bmf:algo-name}
  \small
\begin{algorithmic}[1]
	\StateN{$\LB \leftarrow 1$, $\cost \leftarrow 0$ \Comment{Process the vertices from the stream}}
	\While{there exist unread vertices in $U$}
		\StateN{$C \leftarrow \emptyset$}
		\StateN{$f \leftarrow \LB / (k(1+\lg n))$}
		\For{$u \leftarrow$ next vertex from stream}
			\StateN{$d \leftarrow \min_{c\in C} d(x_u,x_c)$}\label{line:bmf:closest-center}
			\StateN{openCenter $\leftarrow$ True, with
				probability $\min\{w(x) \cdot d/f, 1\}$, and False, otherwise}
			\If{openCenter = True} \Comment{open $u$ as center}
				\State{$C \leftarrow C \cup \{u\}$}
				\State{$w(u) \leftarrow 1$}
			\Else \Comment{Assign $u$ to its closest center $c(u)$}
				\StateN{$\cost \leftarrow \cost + w(u) \cdot d$}
				\StateN{$c(u) \leftarrow \argmin_{c\in C} d(x_u,x_{c(u)})$}
				\StateN{$w(c(u)) \leftarrow w(c(u)) + w(u)$}
				\StateN{$\MG(c(u)) \leftarrow \MG(c(u)) \cup \MG(u)$}
			\EndIf
			\If{$|C| = c_{\max}$ or $\cost>2\LB$}
				\StateN{break and raise flag}
			\EndIf
		\EndFor
		\If{flag raised}
			\StateN{$U \leftarrow $ the stream consisting of the (weighted)
				vertices in $C$ and all unread vertices of $U$}
			\StateN{$\LB \leftarrow 2\LB$}
		\EndIf
	\EndWhile
	\StateN{$(C_1,\dots,C_k) \leftarrow$ clustering of $C$ using an
		$O(1)$-approximate $k$-Medians algorithm
		\Comment{Postprocessing}}\label{line:bmf:offline-clustering}
	\ForAll{$i = 1,\dots,k$}
		\StateN{$\MG_i \leftarrow \bigcup_{x \in C_i} \MG(x)$}\label{line:bmf:sofa-heavy-hitters}
		\StateN{$|C_i| \leftarrow \sum_{c \in C_i} w(c_i)$}
		\StateN{$\tilde{V_i} \leftarrow \{v \in V : $ the counter of $v$ in $\MG_i$ is at
			least $\theta |C_i| \}$}\label{line:bmf:sofa-thresholding}
	\EndFor
\end{algorithmic}
\end{minipage}
\end{algorithm}

\emph{Space Usage and Running Time.} We briefly argue that \sofa's space usage
is $O(k s \lg m)$ and its running time is bounded by $O(m k s \lg m)$. Observe
that the main space usage comes from storing the set of centers $C$ together
with a heavy hitters data structure for each center. Recall that we ensure that
$|C|\leq c_{\max}$ at all times. Furthermore, each center has $O(s)$ incident
edges (by assumption on our input stream) and we set the number of counters for
each heavy hitters data structure to $O(s)$. Thus, the total space usage is
$O(c_{\max} s)$.  Based on the analysis in~\cite{braverman11streaming}, we set
$c_{\max} = O(k \lg m)$ to obtain a constant factor approximation for
$k$-Median and this gives a total space usage of $O(k s \lg m)$.
This also gives us that clustering the vertices requires a
running time of $O(m k s \lg m)$, where we use that merging the heavy
hitters data structures can be done in constant amortized time.
In Section~\ref{sec:bmf:analysis-importance} we sketch how one can obtain
provable guarantees for \sofa.

\emph{Remark.} We use the streaming
$k$-Medians clustering algorithm from~\cite{braverman11streaming}, because the
centers it maintains are points from the stream. Thus, if these points
are sparse, the space usage of
\sofa for storing centers directly benefits from this.  Algorithms for streaming
$k$-Means (e.g.,~\cite{shindler2011fast}) often include steps, which
cause the centers to become dense.
Thus, if we
used such an algorithm as a subroutine, \sofa would require more space.
Here, however, we focused on setting close to the information-theoretically
minimum space usage
and, hence, we decided to use the algorithm by~\cite{braverman11streaming}.

\section{Second Pass: Recover Left Clusters}
\label{sec:bmf:algorithm-left}

In this section, we present algorithms for computing a clustering
$\tilde{U_1},\dots,\tilde{U_k}\subseteq U$ of the left side of the graph during
a second pass over the stream $U$.  We assume that our algorithms obtain as
input a set of clusters $\tilde{V_1},\dots,\tilde{V_k} \subseteq V$ from the
right side of the graph.
We will present two different algorithms for biclustering and BMF, respectively.

\subsection{Biclustering}
\label{sec:bmf:algorithm-left:biclustering}
We now present an algorithm which performs a single pass over the stream $U$ and
assigns each $u\in U$ to exactly one cluster $\tilde{U_i}$. We will use this
algorithm for the biclustering problem, where each vertex $u\in U$ belongs to a
unique planted cluster $U_i$ (see Section~\ref{sec:bmf:preliminaries:biclustering}).

To obtain the clustering $\tilde{U_1},\dots,\tilde{U_k}$, the algorithm
initially sets $\tilde{U_i} = \emptyset$ for all $i = 1,\dots,k$. Now the
algorithm performs a single pass over the stream of left-side vertices $u\in U$.
For each $u$, let $\Gamma(u)$ denote the set of neighbors of $u$ in $V$, i.e.,
$\Gamma(u) = \{ v\in V : (u,v) \in E \} \subseteq V$. Now the algorithm assigns
$u$ to the cluster $\widetilde{U_{i^*}}$ such that the overlap of $\Gamma(u)$
and $\widetilde{V_{i^*}}$ is maximized relative to the size of
$\widetilde{V_{i^*}}$. More concretely, the algorithm computes
\begin{align}
\label{eq:bmf:intersection}
	i^* = \argmax \{\abs{\Gamma(u) \cap \tilde{V_i}}/\abs{\tilde{V_i}} : i=1,\dots,k\}
\end{align} 
and then assigns $u$ to $\widetilde{U_{i^*}}$.

\emph{Space Usage and Running Time.}
Observe that the algorithm uses space $O(m)$ (where $m = |U|$), since for each
vertex $u\in U$, we need to store to which cluster $U_i$ it was assigned.
Furthermore, the running time of the algorithm is $O(mks)$: For each of the $m$
vertices, we need to compute $i^*$ as per Equation~\eqref{eq:bmf:intersection}. Since
we assume that each vertex $u$ has at most $O(s)$ neighbors and that all
$\tilde{V_i}$ have size $O(s)$, it takes time $O(s)$ to compute
$|\Gamma(u) \cap \tilde{V_i}|/|\tilde{V_i}|$ for fixed $i$. Thus, computing
$i^*$ can be done in time $O(ks)$.

\subsection{BMF}
\label{sec:bmf:algorithm-left:bmf}
Next, we present an algorithm, which performs a single pass over the stream and
computes clusters $\tilde{U_1},\dots,\tilde{U_k}$, where every vertex $u\in U$
may be contained in multiple clusters $U_{i_1},\dots,U_{i_T}$. Recall from
Section~\ref{sec:bmf:preliminaries:bmf} that this corresponds to computing a factor
matrix $L$ for the the BMF problem.

Our approach for computing the sets $\tilde{U_i}$ is similar to the greedy
covering scheme used in~\cite{miettinen08discrete}. The main idea is that for
every $u\in U$, we greedily cover the set $\Gamma(u)\subseteq V$ using the
clusters $\tilde{V_1},\dots,\tilde{V_k}$ similar to the classic set cover
problem.  However, unlike in standard set cover, we do allow for some amount of
``overcovering''.
Note that this greedily minimizes the symmetric difference of $\Gamma(u)$ and
the sets $\tilde{V_i}$ used for covering $\Gamma(u)$; thus, also their Hamming
distance is minimized.

Before we present our algorithm, let us first define our score function for the
covering process. For sets $A,X,Y$, we define the \emph{score of $A$ for
covering $X$ given that $Y$ was already covered} as
\(
	\score(A \mid X, Y)
	= \abs{ (X \setminus Y) \cap A }
		- \abs{A \setminus (X \cup Y) }.
\)

To better understand the score function, consider the case that no elements of
$X$ were covered before, i.e., $Y = \emptyset$.
Then $\score(A\mid X,\emptyset) = | X \cap A | - |A \setminus X|$ is the number
of elements in $X$, which get covered by $A$, minus the number of those elements
in $A$, which do not appear in $X$ (these elements ``overcover'' $X$). Now
suppose that $Y\neq\emptyset$, i.e., some elements of $X$ were already covered
before and these elements are stored in the set $Y$. Then the score function
takes this into account by not adding score for elements in $A\cap X \cap Y$
that are in $A$ and $X$, but were already covered before. Also, the score
function does not subtract score for elements in $A$ that are not in $X$, but
which were already overcovered before (and, hence, are in $Y$); more precisely,
it does not subtract score for the elements in $(A \cap Y) \setminus X$.

We now describe our greedy algorithm for computing the clusters $\tilde{U_i}$.
Initially, we set $\tilde{U_i} = \emptyset$ for all $i$.  Now we
perform a single pass over the stream $U$ and for each $u\in U$, we do
the following. We initialize $Y_u = \emptyset$ and, as before, let
$\Gamma(u)$ denote the set of neighbors of $u$ in $V$. Now, while there exists an
$i$ such that $\score(\tilde{V_i} \mid \Gamma(u),Y_u) > 0$, we compute
\begin{align}
\label{eq:bmf:max-score}
	i^* = \argmax_{i=1,\dots,k} \score(\tilde{V_i} \mid \Gamma(u), Y_u).
\end{align}
If $\score(\widetilde{V_{i^*}} \mid \Gamma(u), Y_u) > 0$, we assign $u$ to
$\widetilde{U_{i^*}}$ and we set $Y_u = Y_u \cup \widetilde{V_{i^*}}$.
Otherwise, we stop covering $u$ and proceed
with the next vertex from the stream.

\emph{Space Usage and Running Time.}
The space usage is $O(km)$ since each vertex can be assigned to as many as $k$
clusters. The running time of the algorithm is $O(m k^2 s)$: First, note that
evaluating $\score(\tilde{V_i} \mid \Gamma(u), Y_u)$ takes time $O(s)$ because
all sets have size $O(s)$. Second, for a single iteration of the
while-loop we need to evaluate the score function $O(k)$ times to obtain $i^*$
and there are at most $k$ iterations. Hence, we need to spend time $O(k^2 s)$
for each of the $m$ vertices in $U$.

\section{Implementation}
\label{sec:bmf:implementation}

We implemented the \sofa algorithm from
Section~\ref{sec:bmf:algorithm-right:importance} for recovering the right-side
clusters and the two algorithms from Section~\ref{sec:bmf:algorithm-left} for
recovering the left-side clusters. In this section, we present certain
adjustments that we made to improve the results of the algorithms and we discuss
how to set certain parameters of the algorithms.

We implemented all algorithms in Python. To speed up the computation, the
subroutines for finding the closest centers (Line~\ref{line:bmf:closest-center} in
Algorithm~\ref{algo:bmf:importance}) and for finding the clusters with maximum score
(Equation~\eqref{eq:bmf:max-score}) were implemented in CPython. We did not use any
parallelization, i.e., our implementations are purely single-threaded.
Our code is available
online\footnote{\small\url{\codeurl}}.

\subsection{Asymmetric Weighted Hamming Distance}
\label{sec:bmf:asymmetric}
During preliminary tests of \AlgoName on real-world data, we realized that
\AlgoName picked extremely sparse centers which often only had a single non-zero
entry. This resulted in almost all vertices being assigned to this
particular center (because the Hamming distance of a vertex $u$ to a center with
a single non-zero entry is the degree of $u$ plus/minus $1$ and, due to the low
degrees of the left-side vertices $u$, these distances are usually small) which
made the cluster recovery fail.

Hence, we needed to find a way to promote denser centers. To this end, we
introduce an asymmetric weighted version of the Hamming distance which we define
as follows. Let $c \in C$ be a center maintained by \sofa and let $u$ be a
vertex which needs to be clustered. For each entry $i$ of $x_c$ and $x_u$, we
assign the following costs: If $x_c(i) = x_p(i)$, then the cost is $0$; if
$x_p(i) = 1$ and $x_c(i) = 0$ then the cost is $1$; if $x_p(i) = 0$ and $x_c(i)
= 1$ then the cost is $\alpha < 1$. Now the \emph{asymmetric weighted Hamming
distance} of $c$ and $p$ is simply the sum over the costs for all entries of
$x_c$ and $x_p$.

Note that by setting $\alpha = 1$ the above results in the classic (symmetric)
Hamming distance. Furthermore, setting $\alpha < 1$
promotes denser centers because the case of $x_c(i) = 1$ and $x_u(i) = 0$ is
penalized less than in classic Hamming distance.

For example, consider the vectors $x_{c_1}=(1,1,1,1,0)$, $x_{c_2}=(0,0,0,0,1)$
and $x_u=(1,0,0,0,0)$. In vanilla Hamming distance, $u$ would be assigned to
$c_2$ since their distance is $2$ and the distance of $c_1$ and $p$ is $3$.
With asymmetric weighted Hamming distance and $\alpha=0.1$, $u$ is
assigned to $c_1$ because their distance is $0.3$ and the distance is $u$ and
$c_2$ is $1.1$. Note the assignment of $u$ to $c_1$ instead of $c_2$ is also
much more suitable for the thresholding step in
Line~\ref{line:bmf:sofa-thresholding} of \sofa.

In practice, our experiments showed that setting $\alpha = 0.1$ was a good
choice for all datasets and the performance of our algorithms benefitted heavily
from using asymmetric weighted Hamming distance.

\subsection{Biclustering Algorithm}
\label{sec:bmf:implementation:biclustering}
To solve the biclustering problem from
Section~\ref{sec:bmf:preliminaries:biclustering}, we implemented \sofa
(Algorithm~\ref{algo:bmf:importance}) together with the biclustering algorithm from
Section~\ref{sec:bmf:algorithm-left:biclustering} for recovering the left clusters.
The only adjustment that we made was to use the $k$-Means implementation of
scikit-learn~\cite{scikit-learn}  in order to implement the $O(1)$-approximate
$k$-Medians algorithm in Line~\ref{line:bmf:offline-clustering} of \sofa.

\subsection{BMF Algorithm}
\label{sec:bmf:implementation:bmf}
To solve the BMF problem from Section~\ref{sec:bmf:preliminaries:bmf}, we
implemented \sofa (Algorithm~\ref{algo:bmf:importance}) together with the BMF algorithm
from Section~\ref{sec:bmf:algorithm-left:bmf} for recovering the left clusters.

During preliminary tests we observed that on some datasets we achieved better
results when we completely skipped the $k$-Median algorithm in
Line~\ref{line:bmf:offline-clustering} of \sofa.  Instead, we compute a cluster
$\tilde{V_c}$ for each center $c \in C$. Note that this might lead to more than
$k$ clusters $\tilde{V_c}$ but to at most $c_{\max}$. Then we use the BMF algorithm
from Section~\ref{sec:bmf:algorithm-left:bmf} to compute a cluster $\tilde{U_c}$ for
each of the (potentially more than $k$) clusters $\tilde{V_c}$.  While computing
the clusters $\tilde{U_c}$, we keep track of the total score of each cluster
$\tilde{V_c}$; this can be done by maintaining a counter $s_c$ for each $c \in C$
and increasing $s_c$ by $\score(\tilde{V_c} \mid \Gamma(u), Y_u)$ whenever we
compute $i^*$ in Equation~\eqref{eq:bmf:max-score}. To ensure that our algorithm
only returns $k$ clusters when it finishes, we sort the clusters $\tilde{V_c}$
by their score values $s_c$ in non-increasing order and only keep the $k$
clusters with the highest total scores. This ensures that at the end we only
return $k$ clusters.

While \sofa and the algorithm from Section~\ref{sec:bmf:algorithm-left:bmf} return
clusters $\tilde{U_i}$ and $\tilde{V_i}$ instead of Boolean factor matrices $L$
and $R$ as required for the BMF problem, we can transform
the clusters into factor matrices $L$ and $R$ as discussed in
Section~\ref{sec:bmf:preliminaries:bmf}.  This gives raise to a matrix
$\tilde{B} = L \circ R$ which approximates the biadjacency matrix~$B$ of the
input graph~$G$.

\subsection{Setting the Rounding Threshold $\theta$}
\label{sec:bmf:implementation:theta}
Next, we discuss how to set the rounding threshold $\theta$.

\textbf{A Heuristic for Determining $\theta$.}
The supplemental material of~\cite{neumann18bipartite} presents a heuristic for
setting $\theta$. It essentially works by observing that $\theta$ is a function
of the parameters $p$ and $q$ of the random graph model from
Section~\ref{sec:bmf:preliminaries:biclustering}. Then it performs a grid search
over different values of $p$ and $q$ and picks the pair $(p^*,q^*)$ for which
the resulting rounding threshold $\theta^*$ maximizes the likelihood of the
counters observed in the heavy hitters data structure from
Line~\ref{line:bmf:sofa-heavy-hitters} of \sofa. We refer to the supplemental
material of~\cite{neumann18bipartite} for the details of the heuristic.
We will refer to the version of \sofa which uses this heuristic as \sofaauto.

\textbf{Using Multiple Thresholds.}
Note that the only place in \sofa, where the rounding threshold $\theta$ is
used, is in the postprocessing step. Thus, given multiple rounding thresholds
$\theta_1,\dots,\theta_T$, it is possible to compute a set of clusters
$\tilde{V}_1^{(t)},\dots,\tilde{V}_k^{(t)}$ for each $\theta_t$.  Then for
each $t=1,\dots,T$, we can compute corresponding left-side clusters
$\tilde{U}_1^{(t)},\dots,\tilde{U}_k^{(t)}$ using the algorithms from
Section~\ref{sec:bmf:algorithm-left}. Note that computing the clusters
$\tilde{U}_i^{(t)}$ for all values of $t=1,\dots,T$ still only requires a single
pass over the stream: For each $u\in U$ of the stream, we can run the algorithms
for computing $\tilde{U}_1^{(t)},\dots,\tilde{U}_k^{(t)}$ in parallel for all
$t=1,\dots,T$.

In our experiments we will use the above strategy to generate clusters for
multiple thresholds. Then we will evaluate their quality in a separate
postprocessing step (see Section~\ref{sec:bmf:experiments:real}).  We will refer to
the version of \sofa which uses multiple thresholds simply as \sofa.

\subsection{Static to Streaming Reduction}
\label{sec:bmf:implementation:reduction}
Since many static algorithms do not scale to datasets of the size
considered in this paper, we describe a reduction for turning \emph{static}
biclustering/BMF algorithms into \emph{2-pass streaming} algorithms.  We
will use this reduction to compare \sofa against static algorithms in our
experiments.

In a nutshell, the reduction works as follows. First, we sample a subgraph
with $\tilde{m}$ left-side vertices and $\tilde{n}$ right-side vertices,
where $\tilde{m}\ll m$ and $\tilde{n}\ll n$ are parameters of the reduction.
Then we run the static algorithm on the sampled subgraph to determine a set
of right-side clusters $\tilde{V_1},\dots,\tilde{V_k}$ (see below for
details).  In the second pass over the stream, we use exactly the same
procedure as used by \sofa (see Section~\ref{sec:bmf:algorithm-left}) to infer
the left-side clusters $\tilde{U_1},\dots,\tilde{U_k}$.

Now, we elaborate on the first pass over the stream.  First, we use
reservoir sampling to obtain $\tilde{m}$ left-side vertices from the graph
uniformly at random; let $U'=\{u_1',\dots,u_{\tilde{m}}'\}$ denote this set
of left-side vertices.  Let $V'$ be the set of right-side vertices which are
adjacent to vertices in $U'$.  Note that possibly $\abs{V'}>\tilde{n}$ and let
$V''$ be the set of $\tilde{n}$ vertices in $V'$ with highest degree to
vertices in $U'$ (breaking ties arbitrarily).  Now we run the static
algorithm on the subgraph with the $\tilde{m}$ left-side vertices $U'$ and
$\tilde{n}$ right-side vertices $V''$. This gives raise to clusters
$\tilde{V_1},\dots,\tilde{V_k}$.  Next, we add the (low-degree) vertices
$v\in V'\setminus V''$ to the clusters $\tilde{V_i}$ by assigning each $v$
to the cluster $\tilde{V_i}$ which ``on average'' has the most similar
left-side neighborhood compared to $v$.  More concretely, for each vertex
$v\in V'$ we define the vector $x_v\in\{0,1\}^{\tilde{m}}$ such that
$x_v(i)=1$ iff $(u_i',v)\in E$.  Next, for each cluster $\tilde{V_i}$ define
the vector $x_i=\sum_{v\in\tilde{V_i}}x_v/\abs{\tilde{V_i}}$ which describes the
``average left-side neighborhood'' of the vertices in $\tilde{V_i}$.  Now we
assign each $v\in V'\setminus V''$ to $\tilde{V_{i^*}}$ with $i^*=\argmin_i
d(x_i,x_v)$.  This yields the final clusters
$\tilde{V_1},\dots,\tilde{V_k}$.

\section{Experiments}
\label{sec:bmf:experiments}

We evaluate \sofa on synthetic and on real-world datasets.  We conducted the
experiments on a workstation with 4 Intel~i7-3770 processors at 3.4~GHz and
16~GB of main memory.

\subsection{Synthetic Datasets}
\label{sec:bmf:experiments:synthetic}
We start by evaluating our biclustering version of \sofa from
Section~\ref{sec:bmf:implementation:biclustering} on synthetic data. We ran \sofa
with different numbers of centers $c_{\max} \in \{100,200\}$ and with $100$ and
$200$ counters in the heavy hitters data structures.

We
compare \sofa against three different algorithms. First, a version of the
algorithm from~\cite{neumann18bipartite}
which does not use any spectral preprocessing; this algorithm is denoted
\staticsofa.  \staticsofa can be viewed as a non-streaming version of \sofa,
i.e., it performs the clustering offline using $k$-Means (instead of streaming
$k$-Median) and then it performs the thresholding step
(Line~\ref{line:bmf:sofa-thresholding}) using the exact frequency counts (instead of
the approximate frequency counts from the heavy hitters data structures). Thus,
\staticsofa essentially provides an upper bound on how good the streaming
version of \sofa can potentially get.
Next, we turn the static biclustering algorithms by
Dhillon~\cite{dhillon01coclustering} and Zha et al.~\cite{zha01bipartite}
into streaming algorithms via the reduction from
Section~\ref{sec:bmf:implementation:reduction}, where we set
$\tilde{m}=\tilde{n}=5000$, i.e., we sample subgraphs with 5000 vertices on
both sides. We denote these algorithms \dhillon and \zha, where
$\textsf{RS}$ stands for \emph{random subgraph}.

\textbf{Data Generation and Quality Measure.}
We generated the synthetic data as follows. We start with an empty graph and
then for each ground-truth cluster $U_i$, we insert $\ell$ vertices (see below
for which values of $\ell$ we used in the experiments). Then we inserted
8000 vertices on the right side of the graph (i.e., $\abs{V}=n=8000$). To generate
the ground-truth clusters $V_i$, we simply picked $r$ vertices uniformly at
random from $V$ for each $i$ (see below for how $r$ was set in the experiments).
Now the random edges were inserted exactly as described in the random graph
model from Section~\ref{sec:bmf:preliminaries:biclustering}.

When not mentioned otherwise, we have set the parameters for the graph
generation as follows: $n=8000$, $k=50$, $\ell=200$ (and, hence,
$\abs{U}=m=k\cdot\ell=10\,000$), $p=0.7$, $r=30$. Furthermore, we set $q$ such that
in expectation every left-side vertex obtains $20$ random neighbors. 

To evaluate the output of the algorithms, let $U_1,\dots,U_k$ be the planted
ground-truth clusters and let $\tilde{U_1},\dots,\tilde{U_k}$ be the clusters
returned by one of the algorithms. We define the \emph{quality $Q$} of the
clustering $\tilde{U_1},\dots,\tilde{U_k}$ as 
\begin{align*}
	Q = \frac{1}{k} \sum_{i = 1}^k \max_{j = 1,\dots,s} J(U_i, \tilde U_j) \in [0,1],
\end{align*}
where $J(A,B) = \abs{A \cap B} / \abs{A \cup B}$ is the Jaccard coefficient. That is,
for each ground-truth cluster $U_i$, we find the cluster $\tilde U_j$ which
maximizes the Jaccard coefficient of $U_i$ and $\tilde U_j$. The quality is then
simply the sum over the Jaccard coefficients for all ground-truth clusters
$U_i$, normalized by $k$. Clearly, higher values for $Q$ imply a clustering
closer to the planted clustering.  For example, if the clusters $\tilde U_j$
match \emph{exactly} the ground-truth clusters $U_i$ then $Q=1$.  We evaluate
the quality of the clusters $\tilde{V_i}$ in exactly the same way.

\textbf{Experiments.}
Next, let us discuss the outcomes of our experiments in different scenarios,
where each time we vary one of the parameters. For each set of parameters we
generated 15 different datasets and we will be reporting averages and standard
deviations for the recovery quality of the algorithms.  Our results are reported
in Figure~\ref{fig:bmf:experiments}.

\emph{Varying Amount of Signal.} First, let us consider a varying
amount of signal, i.e., we set $p\in\{0.5,0.6,0.7,0.8,0.9\}$. One can see in
Figures~\ref{fig:bmf:vary_p_figLeft} and~\ref{fig:bmf:vary_p_figRight} that the quality
of all \sofa-versions improves as $p$ increases. Furthermore, \staticsofa
achieves the best quality for recovering the left and right clusters. The
second-best \sofa-version is \sofa with 200 counters and 200 centers and achieves
between 0.05 and 0.1 less quality than \staticsofa; we ran significance tests
and these differences are significant. When only providing 100 centers, \sofa
has some problems for values $p\in\{0.5,0.6\}$; this is not surprising since we
planted $50$ clusters and thus only maintaining $100$ centers is quite restrictive
for \sofa. 
The right-side recovery of \dhillon and \zha is relatively constant, where
\dhillon is performing on a high level; we explain the flatness of the curves by
the spectral methods used in the algorithms, which ``denoise'' the data well
even for small $p$.  The left-side recovery of both algorithms is clearly worse
than those of the \sofa-versions.
Regarding the running times (Figure~\ref{fig:bmf:vary_p_figTimes}), we
see that all versions of \sofa are about a factor $3$ faster than \staticsofa;
note that \sofa with 100 centers is also significantly faster than the versions
of \sofa with 200 centers.
\dhillon and \zha are about factor 1.5--2 slower than \sofa.

\emph{Varying Size of Right Clusters.} Next, we varied the sizes
$r\in\{15,20,30,50\}$ of the planted right clusters $V_i$. We can see
(Figures~\ref{fig:bmf:vary_size_figLeft} and~\ref{fig:bmf:vary_size_figRight}) that most
algorithms benefit from larger $r$ and that once again \staticsofa is the best
method, followed by \sofa with 200~counters and 200~centers.  When the right
clusters are very small (sizes $15$, $20$), \sofa is much worse than \staticsofa 
and \dhillon. Indeed, for small values of $r$, the vertices become
much harder to cluster for \sofa, because the Hamming distances of the vertices
get dominated by noise. However, for $r\geq 30$, the version of
\sofa with 200~counters and 200~centers only has a 0.1~gap in quality compared
to \staticsofa.  Furthermore, observe that the performance of \sofa with only
100~counters in the heavy hitters data structures drops dramatically for $r=50$;
this is caused by the frequency estimations of the right-side vertices getting
too inaccurate due to the too small number of counters in the heavy hitters data
structures.  
\dhillon's quality is again relatively constant at roughly the same
level as before, while \zha clearly benefits from larger cluster sizes.
The
running times of the algorithms (Figure~\ref{fig:bmf:vary_size_figTimes}) slightly
rise as $r$ increases since the datasets contain more non-zero entries.

\emph{Varying Size of Left Clusters.} Finally, we varied the size~$\ell$ of the
left clusters $U_i$ and set $\ell\in\{100,150,200,300,400,500,600\}$.
Note that this implies that we are also varying the number of left-side vertices
of the bipartite graph and, hence, also the total number of edges in the graph.
Figures~\ref{fig:bmf:vary_l_figLeft} and~\ref{fig:bmf:vary_l_figRight} show that
the recovery quality is relatively unaffected from this change in $\ell$ and
that the ranking of the algorithms is as before.
However, note that the running times of \staticsofa increase much
more rapidly than those of the streaming algorithms.
For example, for
$\ell=100$ the running times of \sofa and \staticsofa differ by a factor of less
than 2 but for $\ell=600$ this is already approximately 7.

\emph{Conclusion.} We conclude that \sofa can achieve recovery qualities close
to the static baseline even when its number of centers is only $4k$ and its
number of counters is within factor~4 of the size of the right-side clusters.
Furthermore, \sofa's run-time scales much better than the static baseline's.
While \dhillon delivered good quality for right-side recovery, its left-side
recovery was rather poor. \zha performs badly overall; we blame this on the
data being too sparse, which does not allow the algorithm to find good cuts.

\begin{figure*}[t!]
\makebox[\linewidth][c]{
  \subfloat{
	  \includegraphics[width=4\smallfigwidth]{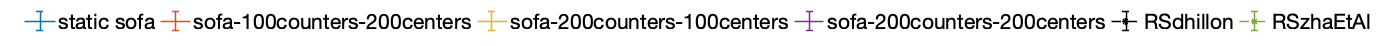}
  }
}
  \addtocounter{subfigure}{-1}

\makebox[\linewidth][c]{
  \subfloat[\small Vary $p$: Left Cluster Quality]{
    \label{fig:bmf:vary_p_figLeft}
    \includegraphics[width=\smallfigwidth]{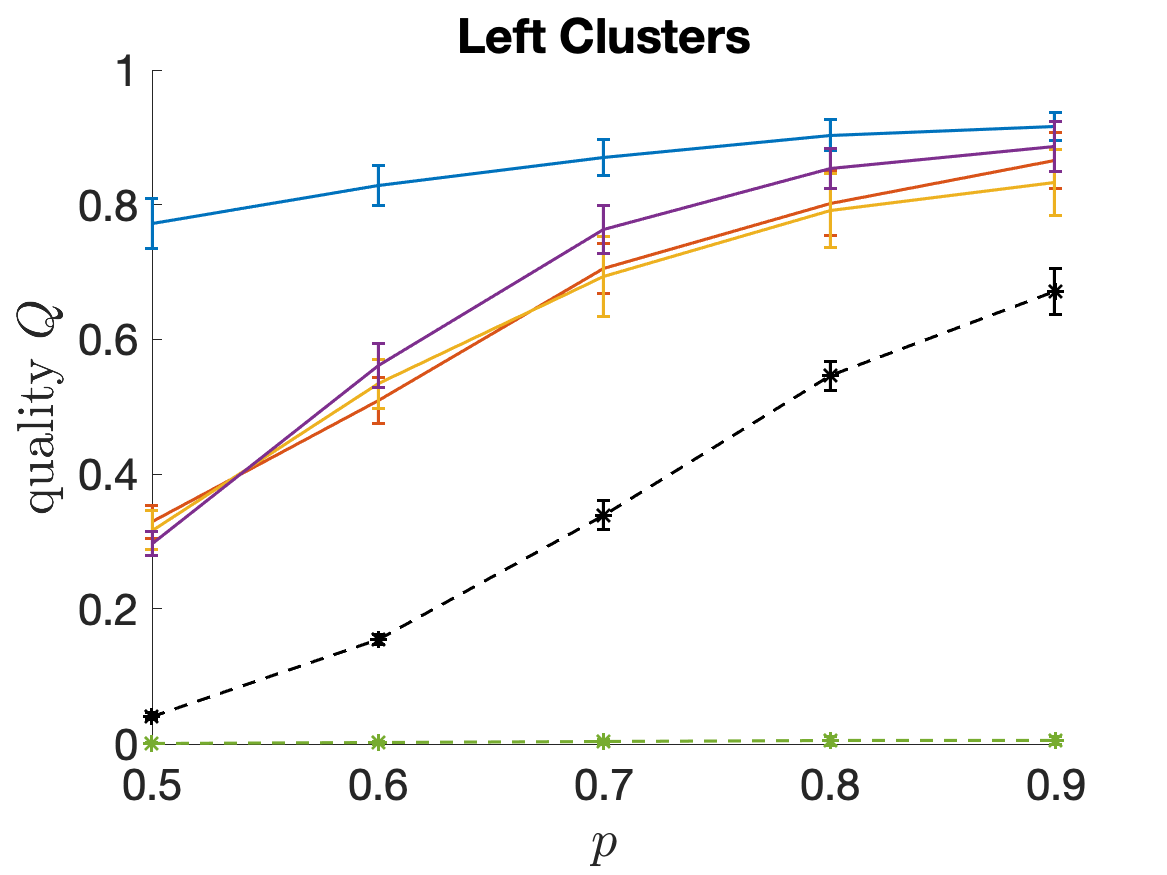}
  }\hspace*{\smallfigsep}
  \subfloat[\small Vary $p$: Right Cluster Quality]{
    \label{fig:bmf:vary_p_figRight}
    \includegraphics[width=\smallfigwidth]{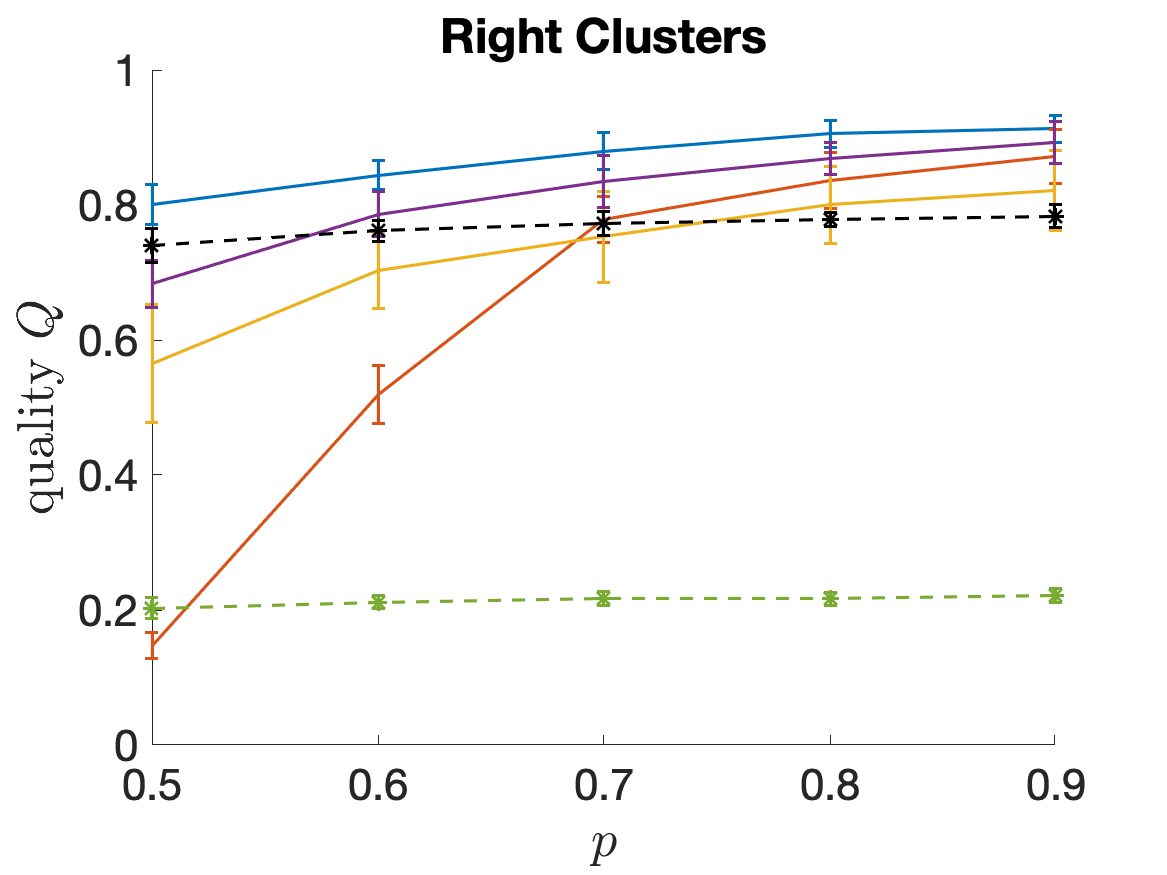}
  }\hspace*{\smallfigsep}
  \subfloat[\small Vary $p$: Running times (sec)]{
    \label{fig:bmf:vary_p_figTimes}
    \includegraphics[width=\smallfigwidth]{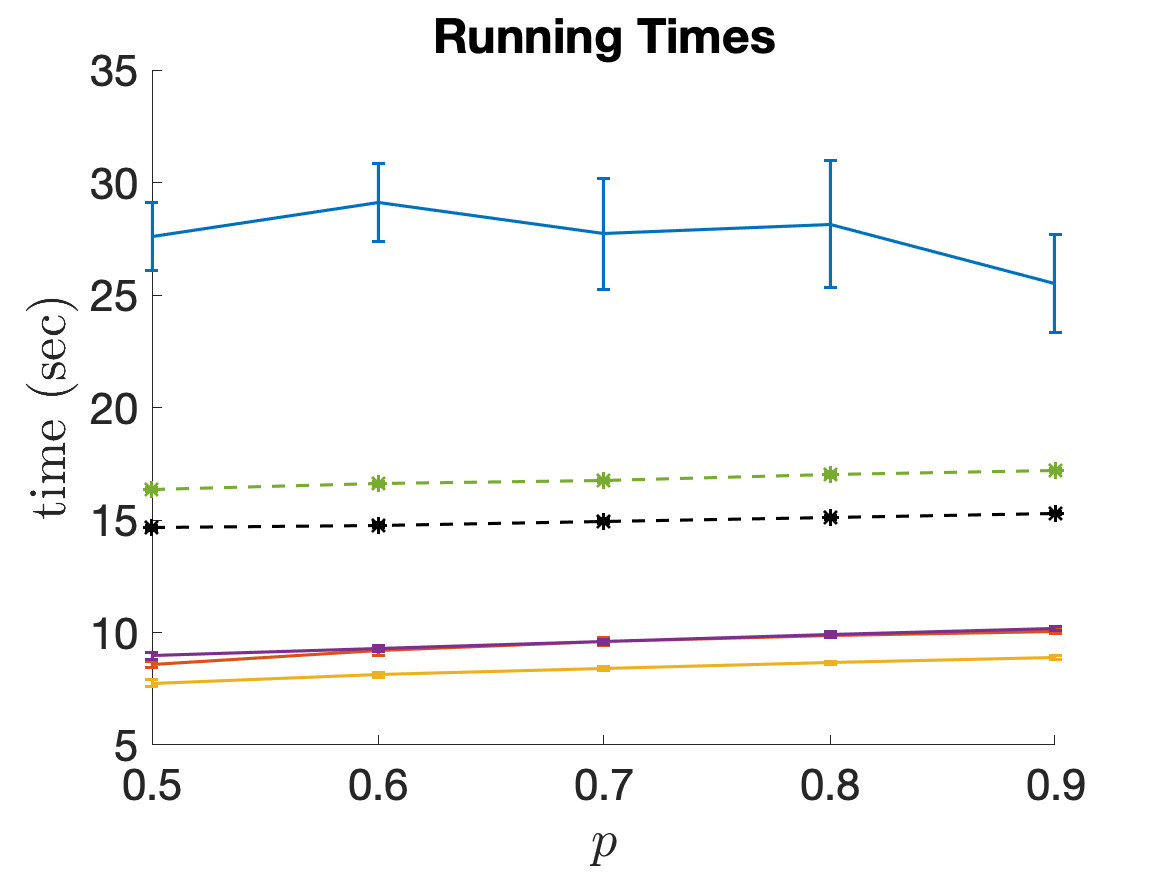}
  }
}

\makebox[\linewidth][c]{
  \subfloat[\small Vary $\abs{V_i}$: Left Cluster Quality]{
    \label{fig:bmf:vary_size_figLeft}
    \includegraphics[width=\smallfigwidth]{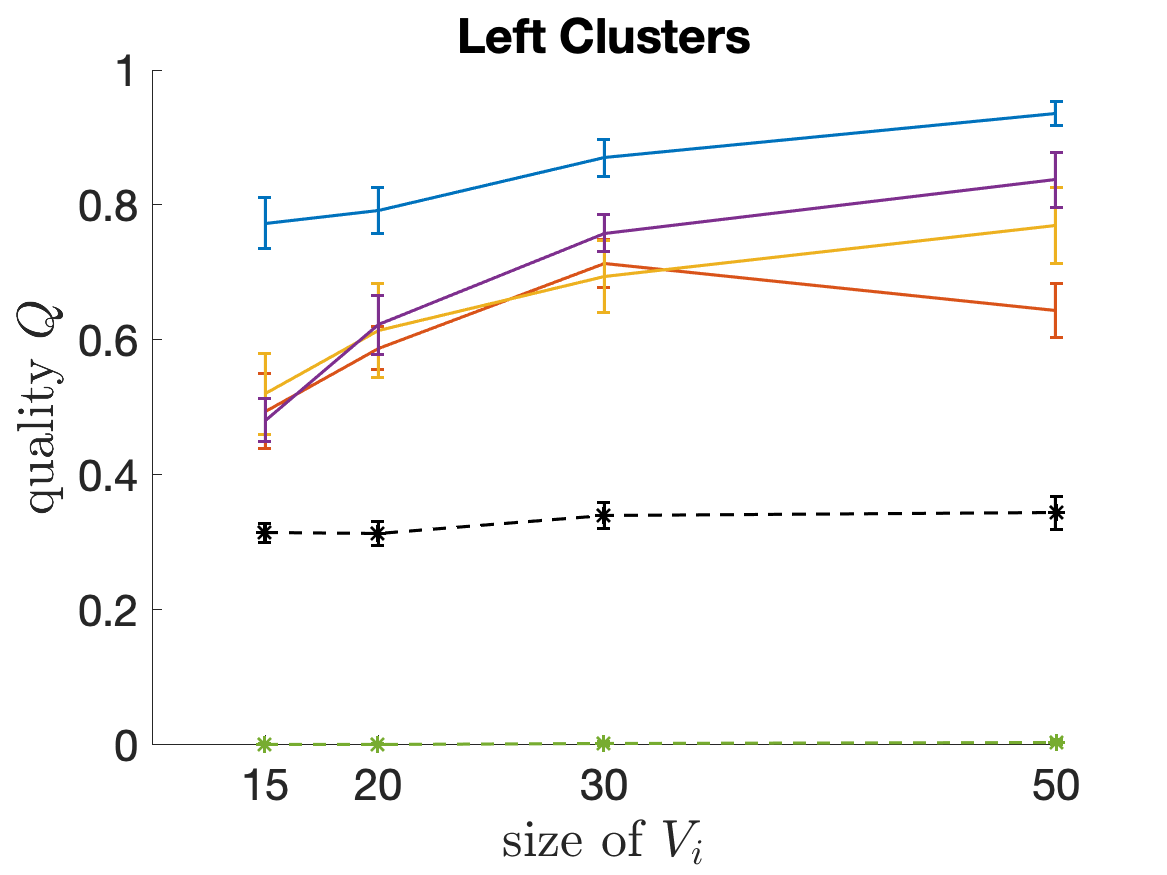}
  }\hspace*{\smallfigsep}
  \subfloat[\small Vary $\abs{V_i}$: Right Cluster Quality]{
    \label{fig:bmf:vary_size_figRight}
    \includegraphics[width=\smallfigwidth]{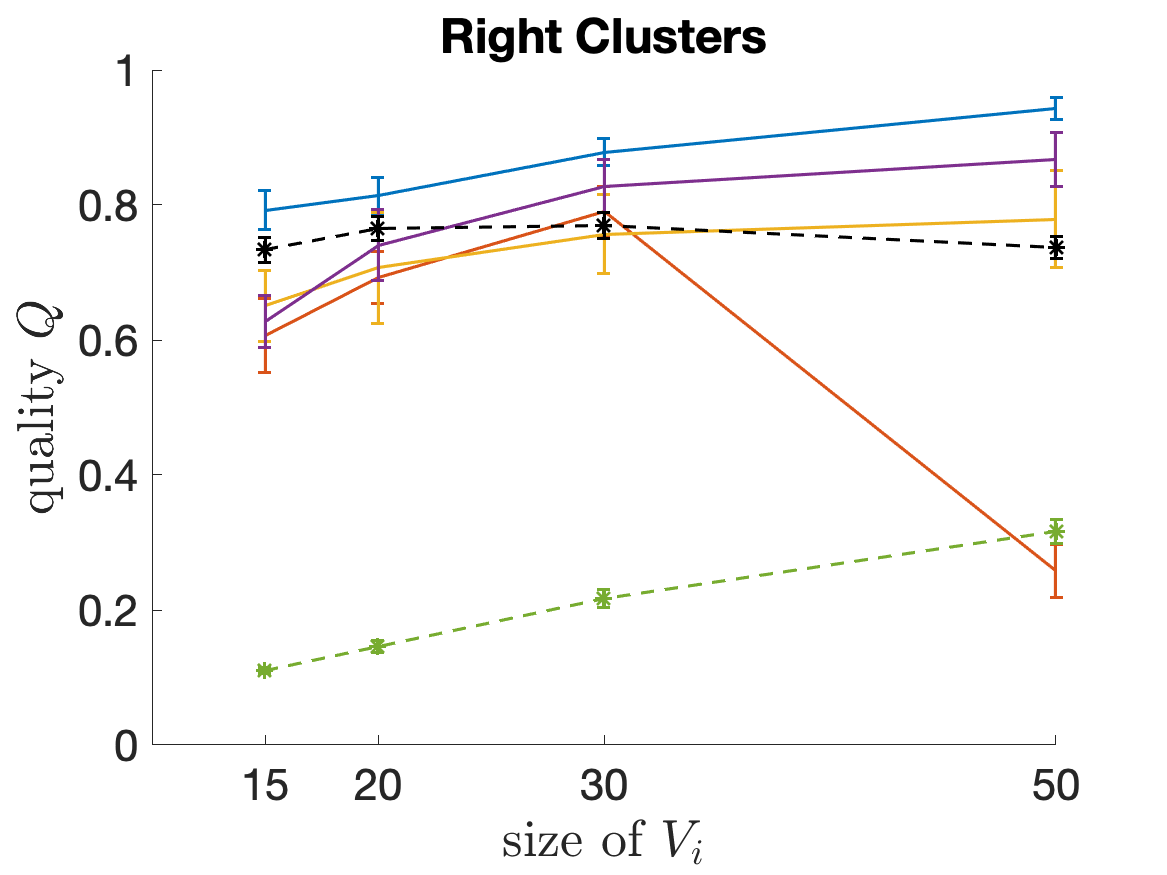}
  }\hspace*{\smallfigsep}
  \subfloat[\small Vary $\abs{V_i}$: Running times (sec)]{
    \label{fig:bmf:vary_size_figTimes}
    \includegraphics[width=\smallfigwidth]{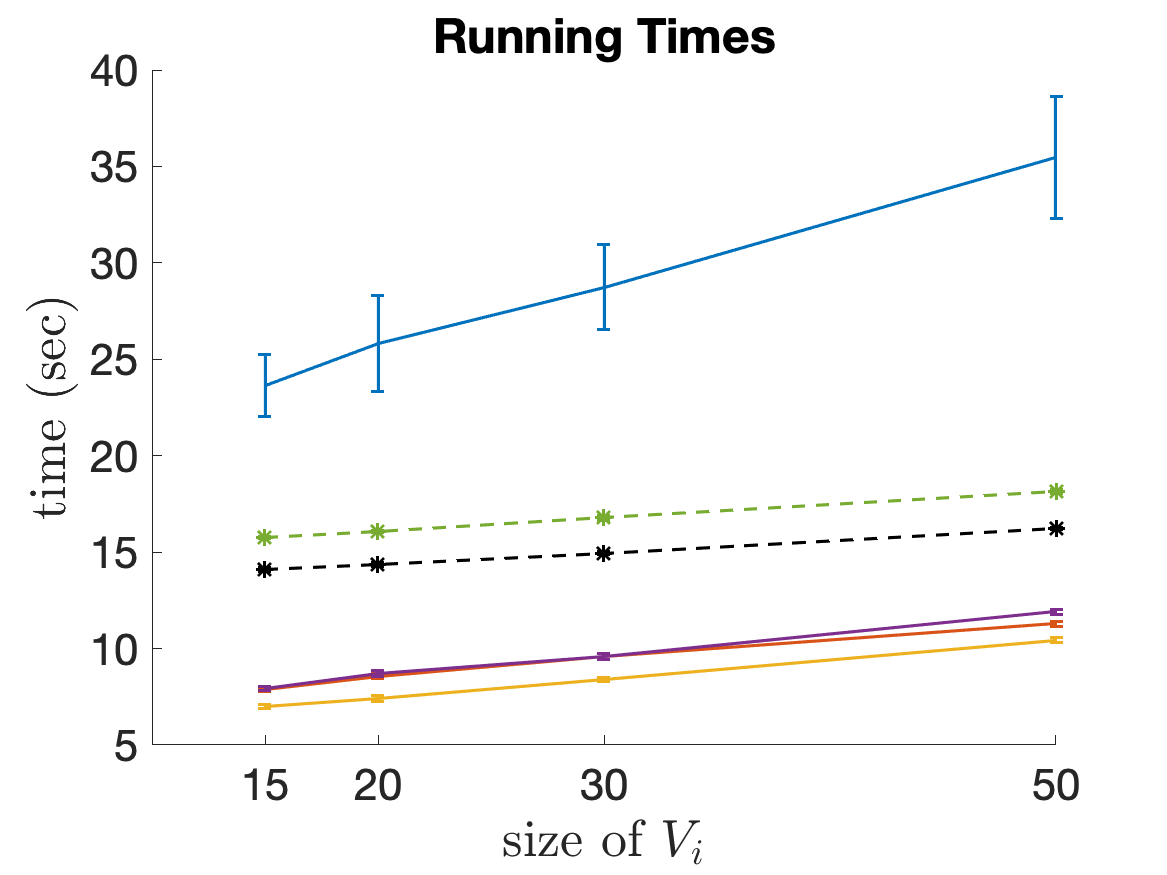}
  }
}

\makebox[\linewidth][c]{
  \subfloat[\small Vary $\abs{U_i}$: Left Cluster Quality]{
    \label{fig:bmf:vary_l_figLeft}
    \includegraphics[width=\smallfigwidth]{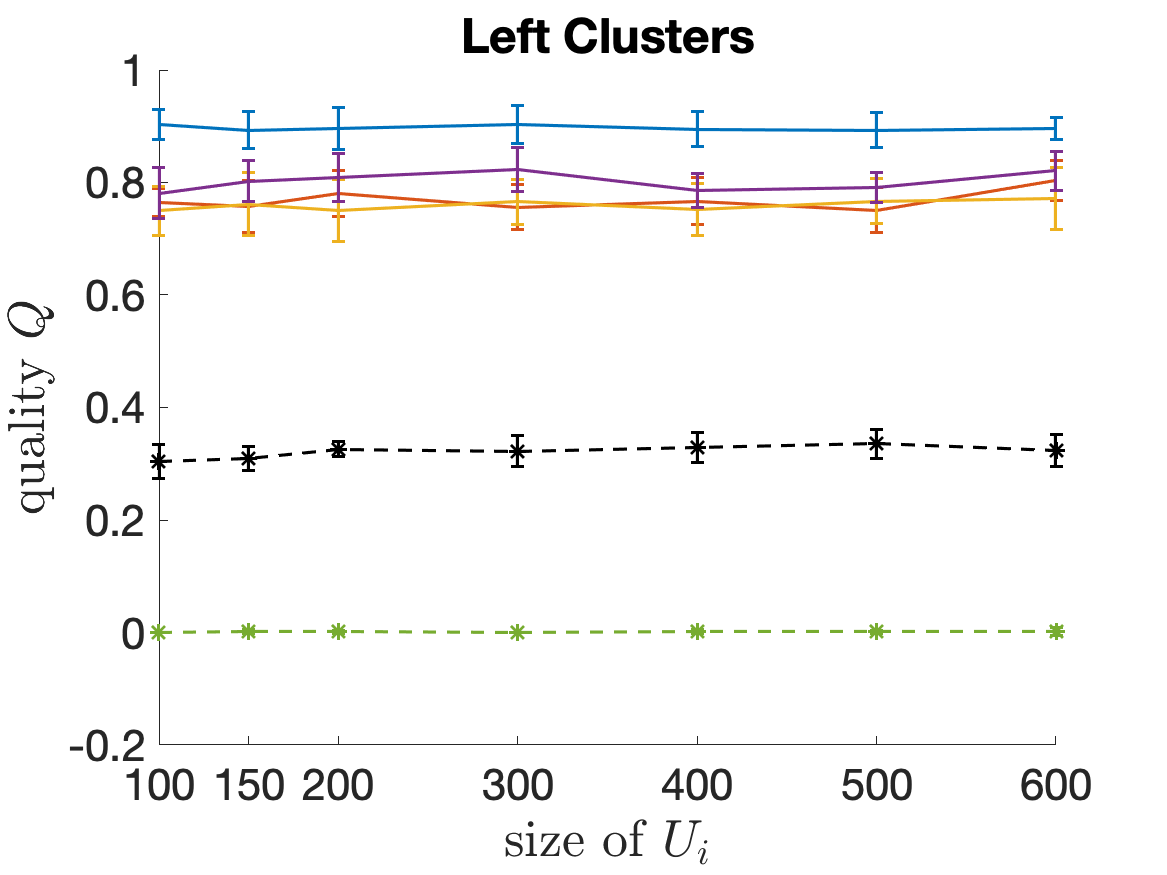}
  }\hspace*{\smallfigsep}
  \subfloat[\small Vary $\abs{U_i}$: Right Cluster Quality]{
    \label{fig:bmf:vary_l_figRight}
    \includegraphics[width=\smallfigwidth]{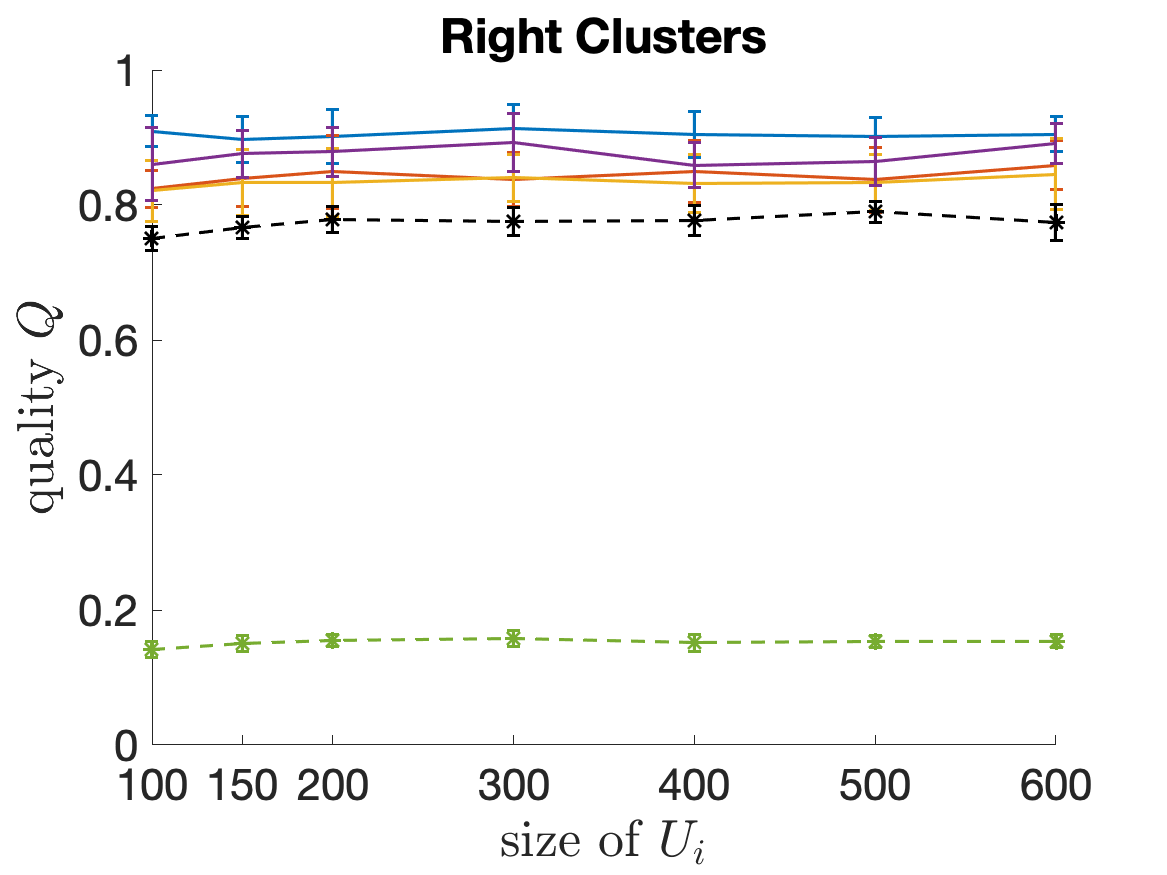}
  }\hspace*{\smallfigsep}
  \subfloat[\small Vary $\abs{U_i}$: Running times (sec)]{
    \label{fig:bmf:vary_l_figTimes}
    \includegraphics[width=\smallfigwidth]{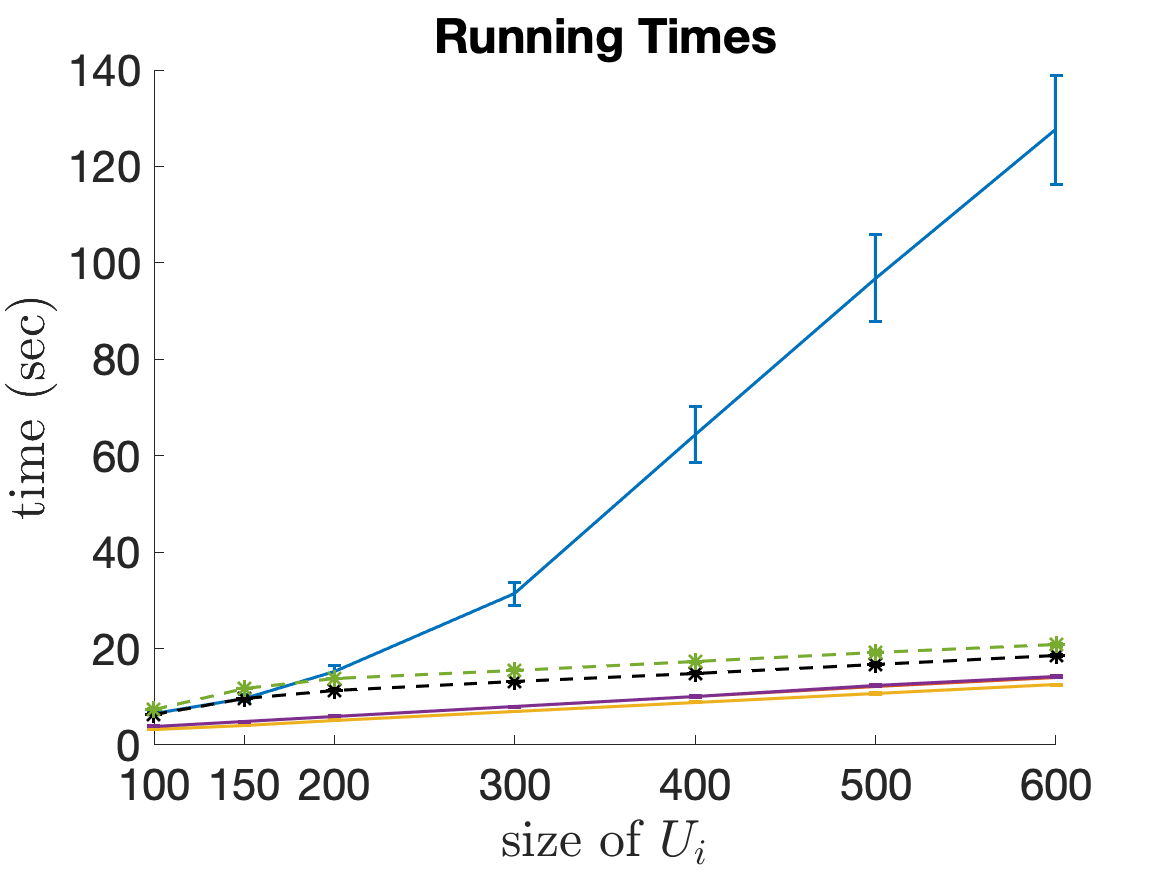}
  }
}
  \caption{ Results on synthetic data.
	  Figures~\ref{fig:bmf:vary_p_figLeft}--\ref{fig:bmf:vary_p_figTimes} have varying $p$,
	  Figures~\ref{fig:bmf:vary_size_figLeft}--\ref{fig:bmf:vary_size_figTimes} have
	  varying sizes of the right clusters $V_i$,
	  Figures~\ref{fig:bmf:vary_l_figLeft}--\ref{fig:bmf:vary_l_figTimes} have varying
		  sizes of the left clusters $U_i$.
	  Markers are mean values over 15 different datasets; error bars are one
	  standard deviation over the 15 datasets. }
  \label{fig:bmf:experiments}
\end{figure*}

\subsection{Real-World Datasets}
\label{sec:bmf:experiments:real}

For the real-world experiments, it is more realistic to allow the left-side
clusters $U_i$ to overlap. Thus, for the real-world experiments, we use the
version of \sofa which solves the BMF problem from
Section~\ref{sec:bmf:implementation:bmf}.

\textbf{Methods and Measures.}
For these experiments, we use \sofa and \sofaauto. For \sofa, we set the
threshold $\theta$ using a line search and we use the values
$\theta\in\{0.3,0.4,0.5,0.6,0.7\}$. The remaining parameters were set as
follows: $c_{\max}=20k$, where $k$ is the desired number of clusters;
$s=P_{99}$, the $99$th quantile of the degrees on the left-side vertices (see
Table~\ref{tab:bmf:exp:properties} for the values for each dataset); and we set the
number of counters in the heavy hitters data structures to $\max\{3s,0.05n\}$.

As for the synthetic datasets, we compare $\sofa$ against \dhillon
and \zha. We used $\tilde{m}=\tilde{n}=15000$ in the reduction. With
these parameters, \dhillon and \zha have running times comparable to \sofa and
already for $\tilde{m}=\tilde{n}=20000$, our workstation would often run out of
memory.
Further, we compare against the static (i.e., non-streaming) algorithm 
\basso\footnote{\basso v0.5 from \small\url{http://cs.uef.fi/~pauli/basso/}}, which is
an efficient implementation of the \asso algorithm~\cite{miettinen08discrete}.
\basso has one hyperparameter, $\tau$. We try values $\tau\in\{0.2, 0.4, 0.6, 0.8\}$ and
report the results with the best value. For run-time and memory usage analysis,
we report average values over different thresholds. The time complexity
of \basso is $O(k\abs{U}^2\abs{V})$ and thus we flipped $U$
and $V$ in the input for \basso when $\abs{U} > \abs{V}$.

For all datasets, we computed clusterings consisting of
$k=50,100,200$ clusters.
Since for the real-world datasets no information about the ground-truth
clusters is available, we use relative Hamming gain and recall as quality
measures to evaluate the obtained clusterings.
Formally, let $B$ be the biadjacency matrix of the bipartite graph and
let $\tilde{B}$ an approximation thereof.
The \emph{relative Hamming gain} is defined as
\(
  1 - \abs{\{(i, j) : B_{ij} \neq \tilde{B_{ij}}\}}/
  \abs{\{(i, j) : B_{ij} = 1\}}
\),
and it indicates how much better $\tilde{B}$ approximates $B$ than a trivial
(all-zeros) matrix would. The \emph{recall} is defined as
\(
  \abs{\{(i, j) : B_{ij} = 1 \land \tilde{B_{ij}} = 1\}}/
  \abs{\{(i, j) : B_{ij} = 1\}} 
\),
and it indicates the fraction of edges (1s) in $B$ which are ``covered''
correctly by the matrix $\tilde{B}$ returned by one of the algorithm.

\textbf{Explanation of Datasets.}
In our experiments, we used six real-world datasets. Their basic properties are
described in Table~\ref{tab:bmf:exp:properties}. Notice that all datasets are very
sparse, and their left-side degrees (even in the $99$th percentile) are small
compared to the number of vertices on the right side of the graph.  This
empirically validates two of the three properties we discussed in the
introduction.

\begin{table*}[t]
  \small
  \centering
  \caption{Real-world dataset properties. Datasets are considered as bipartite
	  graphs $G=(U\cup V, E)$ and density is $\abs{E}/(\abs{U}\cdot \abs{V})$.
	  Average degree $\overline{\text{deg}}$ and the $99$th percentile degree
	  $P_{99}$ are calculated from $U$ and rounded to the nearest integer.}
  \label{tab:bmf:exp:properties}
  \begin{tabular}{@{}lRRRRRR@{}}
    \toprule
    Dataset & \abs{U} & \abs{V} & \abs{E} &  $density$ & \overline{\text{deg}} & P_{99} \\
    \midrule
    \News        &       18\,773 &        61\,056 &  1\,766\,780 &  0.0015 &  94 &    548 \\
    \Reuters     &       38\,677 &        19\,757 &       978\,446 & 0.0013 &   25 &   498 \\
    \Book         &     105\,282 &      340\,550 &  1\,149\,779 &<0.0001 &  11 &    174 \\
    \Movielens &     138\,493 &        26\,744 & 20\,000\,263 & 0.0054 & 144 & 1113 \\
    \Flickr        &     395\,979 &      103\,631 &   8\,545\,307 & 0.0002 &  22 &   268 \\
    \Wikipedia & 1\,562\,433 & 1\,170\,854 & 19\,753\,078 &<0.0001 &  17 &   177 \\
    \bottomrule
  \end{tabular}
\end{table*}

Let us briefly discuss each of the datasets.
\News\footnote{\small\url{http://qwone.com/~jason/20Newsgroups/}} contains newsgroup
postings on the left side and words on the right side; edges indicate a word
appearing in a posting.
The datasets \Reuters and \Flickr were taken from the
KONECT\footnote{\small\url{http://konect.uni-koblenz.de}}~\cite{kunegis13konect}
website.  \Reuters has articles from the news organization Reuters on the left
side and words on the right. \Flickr encodes the group memberships (right) of
Flickr users (left).
\Wikipedia\footnote{\small\url{https://www.cise.ufl.edu/research/sparse/matrices/Gleich/wikipedia-20051105}}
is from the SuiteSparse Matrix Collection\cite{davis11university} and consists
of Wikipedia pages on both sides of the graph; an edge $(u,v)$ indicates that
page $u$ links to page $v$ (note that this relationship is asymmetric).
\Book\footnote{\small\url{http://www2.informatik.uni-freiburg.de/\~cziegler/BX/}}\cite{ziegler05improving}
is a rating matrix consisting of users on the left side and books on the right
side; an edge indicates that a user rated book.
\Movielens\footnote{\small\url{https://grouplens.org/datasets/movielens/20m/}} is a
rating matrix between users and movies~\cite{harper16movielens}.

\textbf{Experiments.}
Results for relative Hamming gain and recall are presented in
Tables~\ref{tab:bmf:exp:hamming_gain} and~\ref{tab:bmf:exp:recall}, respectively. Note that \basso did not finish on the
\Wikipedia dataset, because it ran out of memory.

\begin{table*}[t]
  \small
  \centering
  \caption{Relative Hamming gain different real-world datasets}
  \label{tab:bmf:exp:hamming_gain}
  \begin{tabular}{@{}RlRRRRRR@{}}
  \toprule
  k & Algorithm & \multicolumn{6}{c}{Relative Hamming gain} \\
  & & $\News$ & $\Reuters$ & $\Book$ & $\Movielens$ & $\Flickr$ & $\Wikipedia$ \\
  \midrule
   50 & \sofaauto & 0.0298 & 0.0450 & 0.0198 & 0.0805 & 0.0380 & 0.0617 \\
      & \sofa & 0.0424 & 0.0454 & 0.0212 & 0.1188 & 0.0453 & 0.0695 \\
      & \basso & 0.0545 & 0.1005 & 0.1226 & 0.1394 & 0.0719 & -\\
      & \dhillon & 0.0042 & 0.0273 & 0.0008 & 0.1056 & 0.0040 & 0.0001 \\
      & \zha & 0.0001 & 0.0274 & 0.0008 & 0.0297 & 0.0000 & 0.0000 \\
  \midrule
  100 & \sofaauto & 0.0411 & 0.0792 & 0.0298 & 0.1028 & 0.0486 & 0.0730 \\
      & \sofa & 0.0574 & 0.0777 & 0.0333 & 0.1367 & 0.0668 & 0.0824 \\
      & \basso & 0.0793 & 0.1097 & 0.1783 & 0.1739 & 0.1068 & -\\
      & \dhillon & 0.0059 & 0.0307 & 0.0028 & 0.1378 & 0.0137 & 0.0262 \\
      & \zha & 0.0006 & 0.0342 & 0.0030 & 0.0696 & 0.0000 & 0.0000\\
  \midrule
  200 & \sofaauto & 0.0624 & 0.1253 & 0.0427 & 0.1247 & 0.0663 & 0.0861\\
      & \sofa & 0.0930 & 0.1254 & 0.0472 & 0.1598 & 0.0817 & 0.1061 \\
      & \basso & 0.1171 & 0.1334 & 0.2531 & 0.2376 & 0.1556 & -\\
      & \dhillon & 0.0092 & 0.0402 & 0.0024 & 0.1771 & 0.0203 & 0.0270\\
      & \zha & 0.0014 & 0.0291 & 0.0017 & 0.1104 & 0.0007 & 0.0001 \\
  \bottomrule
\end{tabular}

\end{table*}

The results for relative Hamming gain in Table~\ref{tab:bmf:exp:hamming_gain} show that, when it is able to
finish, \basso is always the best method. This is to be expected as it can make
unlimited passes over the data. 
On all datasets except \Book and for all values of $k$, the results of \sofa and
\basso are within factor at most $2.2$. For $k=200$, the results of \sofa are at
most 50\% worse than those of \basso on \News, \Reuters and \Movielens.
With \Book, on the other hand, \sofa is significantly worse (up to factor~5.8)
but still much better than \dhillon and \zha. We believe this results from \Book
being too sparse; indeed, the 50\% percentile of the degrees of the left
vertices in book is $1$ and thus \sofa's clustering seems to fails. Overall, the
results of \sofa and \sofaauto improve significantly as $k$ increases, showing
that it can be used for small and large values of $k$ alike.  \dhillon and \zha
perform well when $|V|$ is small (e.g., \Movielens and \Reuters), but as soon
as $|V|$ increases, their results decays dramatically (e.g., \Book, \Flickr and
\Wikipedia); this appears to be a limitation of the random sampling approach.

\begin{table*}[t]
  \small
  \centering
  \caption{Recall in different real-world datasets}
  \label{tab:bmf:exp:recall}
  \begin{tabular}{@{}RlRRRRRR@{}}
  \toprule
  k & Algorithm & \multicolumn{6}{c}{Recall} \\
  & & $\News$ & $\Reuters$ & $\Book$ & $\Movielens$ & $\Flickr$ & $\Wikipedia$  \\
  \midrule
   50 & \sofaauto & 0.0446 & 0.0649 & 0.0201 & 0.1262 & 0.0480 & 0.0657 \\
      & \sofa & 0.0483 & 0.0652 & 0.0214 & 0.1779 & 0.0474 & 0.0700 \\
      & \basso & 0.0683 & 0.1677 & 0.1226 & 0.2855 & 0.0760 & - \\
      & \dhillon & 0.0069 & 0.0316 & 0.0009 & 0.1999 & 0.0088 & 0.0001 \\
      & \zha & 0.0004 & 0.0447 & 0.0014 & 0.0614 & 0.0001 & 0.0000 \\
  \midrule
  100 & \sofaauto & 0.0570 & 0.0991 & 0.0307 & 0.1597 & 0.0636 & 0.0777 \\
      & \sofa & 0.0649 & 0.0987 & 0.0341 & 0.2030 & 0.0721 & 0.0840 \\
      & \basso & 0.0959 & 0.1907 & 0.1783 & 0.3143 & 0.1124 & - \\
      & \dhillon & 0.0103 & 0.0430 & 0.0060 & 0.2400 & 0.0246 & 0.0302 \\
      & \zha & 0.0017 & 0.0500 & 0.0040 & 0.1182 & 0.0002 & 0.0000 \\
  \midrule
  200 & \sofaauto & 0.0788 & 0.1441 & 0.0435 & 0.1926 & 0.0837 & 0.0924 \\
      & \sofa & 0.0991 & 0.1442 & 0.0479 & 0.2353 & 0.0906 & 0.1087 \\
      & \basso & 0.1321 & 0.2100 & 0.2532 & 0.3521 & 0.1603 & - \\
      & \dhillon & 0.0159 & 0.0619 & 0.0030 & 0.2812 & 0.0317 & 0.0299 \\
      & \zha & 0.0022 & 0.0454 & 0.0027 & 0.1644 & 0.0021 & 0.0002 \\
  \bottomrule
\end{tabular}

\end{table*}

The results concerning the recall in Table~\ref{tab:bmf:exp:recall} look very similar to relative Hamming gain:
For all datasets except \Book, \sofa has approximately $50\%$ of the recall of \basso, and in
\Book it is again significantly worse. For \Wikipedia, \sofa has results that
are comparable to other datasets, thus, the size of \Wikipedia datasets does not
seem to affect the quality.  For \dhillon and \zha we observe a similar behavior
as above.

Using the heuristic in \sofaauto to set the threshold typically leads to
slightly worse results than setting it using line search. Given that the
heuristic is usually 3--4~times as fast, there seems to be a tradeoff which
version one should pick.

\begin{table*}[t]
  \small
  \centering
  \caption{Algorithm run-time on different real-world datasets}
  \label{tab:bmf:exp:time}
  \begin{tabular}{@{}RlRRRRRR@{}}
  \toprule
  k & Algorithm & \multicolumn{6}{c}{Run-time in CPU minutes} \\
  & & $\News$ & $\Reuters$ & $\Book$ & $\Movielens$ & $\Flickr$ & $\Wikipedia$ \\
  \midrule
   50 & \sofaauto & 2.1 & 3.2 & 1.7 & 45.9 & 9.7 & 14.1 \\
      & \sofa & 6.2 & 10.3 & 5.5 & 120.0 & 24.0 & 42.9 \\
      & \basso & 22.7 & 13.2 & 2951.8 & 598.1 & 4667.8 & -\\
      & \dhillon & 28.1 & 23.1 & 16.4 & 27.8 & 21.0 & 49.7\\
      & \zha & 36.0 & 75.2 & 75.4 & 35.9 & 98.5 & 76.3 \\
  \midrule
  100 & \sofaauto & 5.2 & 8.3 & 4.7 & 102.2 & 19.9 & 25.8 \\
      & \sofa & 15.6 & 25.4 & 16.5 & 311.6 & 52.7 & 70.4 \\
      & \basso & 24.6 & 13.6 & 3003.8 & 932.3 & 5066.0 & - \\
      & \dhillon & 26.9 & 23.7 & 18.1 & 31.2 & 23.0 & 55.5 \\
      & \zha & 41.6 & 81.2 & 80.7 & 39.7 & 172.3 & 63.7 \\
  \midrule
  200 & \sofaauto & 12.2 & 34.8 & 14.2 & 229.1 & 63.7 & 57.1 \\
      & \sofa & 43.5 & 142.8 & 60.4 & 959.0 & 161.4 & 157.5 \\
      & \basso & 26.7 & 14.3 & 3097.4 & 1441.2 & 5574.1 & - \\
      & \dhillon & 25.3 & 23.1 & 20.8 & 42.2 & 25.8 & 68.3 \\
      & \zha & 39.4 & 90.0 & 68.6 & 51.5 & 350.8 & 100.9 \\
  \bottomrule
\end{tabular}

\end{table*}

The running times of the algorithms are presented in Table~\ref{tab:bmf:exp:time}.
For \sofa and \sofaauto, presented is the total running time (with full line
search in \sofa); for \basso, the presented time is the \emph{average time for
a single value} of the threshold parameter $\tau$. Still, \basso is consistently
the slowest method, often by orders of magnitude.
The run-times of \dhillon and \zha scale well in $k$, since the size of the
sampled subgraph and, hence, the time spent on the static computation, is
largely unaffected by the choice of $k$.

\begin{table*}[t]
  \small
  \centering
  \caption{Algorithm memory usage on different real-world datasets}
  \label{tab:bmf:exp:memory}
  \begin{tabular}{@{}RlRRRRRR@{}}
  \toprule
  k & Algorithm & \multicolumn{6}{c}{Memory in GB} \\
  & & $\News$ & $\Reuters$ & $\Book$ & $\Movielens$ & $\Flickr$ & $\Wikipedia$ \\
  \midrule
   50 & \sofaauto & 0.15 & 0.12 & 0.10 & 0.24 & 0.21 & 0.20 \\
      & \sofa &  0.16 & 0.13 & 0.10 & 0.24 & 0.20 & 0.22 \\
      & \basso & 0.40 & 0.66 & 10.81 & 1.80 & 11.48 & - \\
      & \dhillon & 8.95 & 8.70 & 6.12 & 8.99 & 7.16 & 5.61 \\
      & \zha & 10.72 & 10.43 & 7.26 & 10.73 & 8.63 & 6.57 \\
  \midrule
  100 & \sofaauto & 0.19 & 0.14 & 0.11 & 0.33 & 0.27 & 0.30 \\
      & \sofa & 0.20 & 0.17 & 0.13 & 0.33 & 0.26 & 0.30 \\
      & \basso & 0.40 & 0.67 & 10.95 & 1.80 & 11.79 & - \\
      & \dhillon & 8.96 & 8.70 & 6.09 & 8.99 & 7.20 & 5.54 \\
      & \zha & 10.71 & 10.40 & 7.26 & 10.73 & 8.58 & 6.63 \\
  \midrule
  200 & \sofaauto & 0.25 & 0.18 & 0.13 & 0.49 & 0.36 & 0.43 \\
      & \sofa & 0.26 & 0.22 & 0.17 & 0.50 & 0.36 & 0.42 \\
      & \basso & 0.40 & 0.67 & 10.99 & 1.80 & 12.22 & - \\
      & \dhillon & 8.96 & 8.68 & 6.00 & 8.98 & 7.18 & 5.57 \\
      & \zha & 10.72 & 10.46 & 7.30 & 10.73 & 8.54 & 6.63 \\
  \bottomrule
\end{tabular}

\end{table*}

The memory usages of the algorithms are presented in Table~\ref{tab:bmf:exp:memory}.
\basso again needs significantly more resources.
\sofa and \sofaauto can compute clusterings of graphs with millions of vertices
and edges, while never using more than 500~MB of RAM.  \dhillon and \zha have
relatively large memory footprints (using gigabytes of memory) due to the
spectral methods they use.

Overall, the real-world experiments show that \sofa can achieve results that are
not too far from a static baseline method, while using only a fraction of
resources.

\section{Theoretical Guarantees}
\label{sec:bmf:theory}
\label{sec:bmf:analysis-greedy}
\label{sec:bmf:analysis-importance}

We prove the theoretical guarantees of our algorithms.

\subsection{Proof of Theorem~\ref{thm:bmf:algo-toy}}
For all proofs we assume that the conditions from
Theorem~\ref{thm:bmf:algo-toy} hold. The concrete values of the constants
$K_j$ are set inside the proofs. 
We start by characterising the distances of vertices from the same
cluster $U_i$ and vertices from different clusters $U_i$ and $U_{i'}$.

\begin{lemma}
\label{lem:bmf:distances}
	Let $u,u' \in U_i$ and let $u'' \in U_{i'}$ for $i' \neq i$. Then with
	probability at least $1-m^{-3}$,
	\begin{align*}
		d(x_u,&x_{u'}) < 1.01 \left[ 2 |V_i| p(1-p) + 2 (|V \setminus V_i|) q(1-q) \right], \\
		d(x_u,&x_{u''}) > 0.99 \large[ |V_i \triangle V_{i'}|(p(1-q) + q(1-p)) \\
							&+ 2 |V_i \cap V_{i'}| p(1-p)
							+ 2 |V \setminus (V_i \cup V_{i'})| q(1-q) \large].
	\end{align*}
\end{lemma}
\begin{proof}
	First, recall that the neighbors of $u, u'$ and $u''$ are random variables
	such that if $u\in U_i$ then
	\begin{align*}
		\Prob{ (u,v_j) \in E}
		=
		\begin{cases}
			p, & v_j \in V_i, \\
			q, & v_j \in V \setminus V_i.
		\end{cases}
	\end{align*}
	Since $u$'s neighbors are random this implies that the vector $x_u$ is a
	random vector with $\Prob{x_u(j) = 1} = \Prob{ (u,v_j) \in E}$. Next,
	observe that we can rewrite the event $\{ x_u(j) \neq x_{u'}(j) \}$ 
	as $\{ x_u(j) = 1 \text{ and } x_{u'}(j) = 0 \} \cup \{ x_u(j) = 0 \text{ and } x_{u'}(j) = 1 \}$.
	Together, this implies for vertices from the same cluster,
	\begin{align*}
		\Prob{ x_u(j) \neq x_{u'}(j) } =
		\begin{cases}
			2 p (1-p), & v_j \in V_i, \\
			2 q (1-q), & v_j \in V \setminus V_i.
		\end{cases}
	\end{align*}
	Similarly, we obtain for vertices from different clusters,
	\begin{align*}
		&\Prob{ x_u(j) \neq x_{u''}(j) } = 
		\begin{cases}
			p (1-q) + q(1-p), & v_j \in V_i \triangle V_{i'}, \\
			2 p (1-p), & v_j \in V_i \cap V_{i'}, \\
			2 q (1-q), & v_j \not\in V_i \cup V_{i'}.
		\end{cases}
	\end{align*}
	Next, using linearity of expectation we get that
	\begin{align*}
		\Exp{ d(x_u,x_{u'})  }
			=& \sum_{j=1}^n \Prob{ x_u(j) \neq x_{u'}(j)} \\
			=& 2 |V_i| p(1-p) + 2 (|V \setminus V_i|) q(1-q), \\
		\Exp{ d(x_u,x_{u''}) }
			=& |V_i \triangle V_{i'}|(p(1-q) + q(1-p)) \\
				&+ 2 |V_i \cap V_{i'}| p(1-p) \\
				&+ 2 |V \setminus (V_i \cup V_{i'})| q(1-q).
	\end{align*}

	Since $|V_i|\geq K_3 \lg n$ and
	$|V_i \triangle V_{i'}| \geq K_4 s \geq K_3 K_4 \lg n$,
	\begin{align*}
		\Exp{ d(x_u,x_{u'})  } &\geq 2 p(1-p) |V_i|
			\geq 2 K_3 p(1-p) \lg n, \\
		\Exp{ d(x_u,x_{u''}) }
			&\geq |V_i \triangle V_{i'}|(p(1-q) + q(1-p)) \\
			&\geq K_4 p(1-q) s
			\geq K_3 K_4 p(1-q) \lg n.
	\end{align*}
	A Chernoff bound and setting $K_3$ large enough implies the
	lemma (we will set $K_4$ later independently of $K_3$).
\end{proof}

Next, we show that when setting $\alpha=0.49 K_4 s$ in
Algorithm~\ref{algo:bmf:greedy}, the algorithm clusters all left-side vertices
correctly.
\begin{lemma}
\label{lem:bmf:correct-clustering}
	The following events hold w.h.p.: (1)~When
	Algorithm~\ref{algo:bmf:greedy} finishes, $|C| = k$ and for all $i$, $C$ contains exactly one
	center $c$ with $c\in U_i$.  (2)~For all $i$, there exists a center $c_i
	\in C$ s.t.\ all points $u \in U_i$ were assigned to~$c_i$. 
\end{lemma}
\begin{proof}
	First, we condition on the event from Lemma~\ref{lem:bmf:distances} occurring
	for each pair of vertices from $U$ for the rest of the proof. A union bound
	implies that this happens with probability at least $1-m^{-1}$.

	Second, consider $u, u' \in U_i$. Then
	\begin{align*}
		d(x_u,x_{u'})
		&< 1.01 \left[2 s p(1-p) + 2n q (1-q) \right] \\
		&\leq 1.01 \left[s/2 + 2n \frac{K_1 s}{n} \right] 
		\leq 1.01 (1/2 + 2K_1) s,
	\end{align*}
	where we used $p(1-p) \leq 1/4$ and $q \leq K_1 p s/n \leq K_1 s/n$.

	Third, for $u \in U_i$ and $u'' \in U_{i'}$ for $i \neq i'$,
	\begin{align*}
		d(x_u, x_{u''})
		&> 0.99 \large[ |V_i \triangle V_{i'}|(p(1-q) + q(1-p)) \\
							&\,\,\,\, + 2 |V_i \cap V_{i'}| p(1-p) 
							+ 2 |V \setminus (V_i \cup V_{i'})| q(1-q) \large] \\
		&\geq 0.99 \large[ K_4 s p (1-q) + 0 + 0 \large]
		\geq 0.98 K_4 s/2,
	\end{align*}
	where we used that $|V_i \triangle V_{i'}| \geq K_4 s$ and further
	$p(1-q) \geq p - K_1 p^2 s / n \geq p - K_1 p^2 \geq \frac{0.98}{0.99} \cdot \frac{1}{2}$,
	since $p\geq 1/2$ and since we can pick $K_1$ small enough to satisfy the
	last inequality.
	
	Pick $K_1,K_4$ with $K_4 \geq \frac{2.02}{0.98} (1/2 + 2K_1)$.
	Then $d(x_u, x_{u''})
		> 0.98 K_4 s/2
		\geq 1.01 (1/2 + 2K_1) s
		>d(x_u,x_{u'})$.

	Next, we show that Algorithm~\ref{algo:bmf:greedy} satisfies the properties of
	the lemma with $\alpha=0.98 K_4 s /2$.
	To prove~(1), suppose a vertex $u\in U_i$ is processed and for all
	$c\in C$, $d(x_u, x_c) > \alpha$. Then $C$ cannot contain any
	point $u'\in U_i$ (if $C$ contained such a point, then the previous
	computation and the event we conditioned on imply $d(x_u,x_{u'}) \leq
	\alpha$). Thus, opening $u$ as a new center is the correct choice and
	$C$ contains exactly one center from $U_i$.  To prove~(2), suppose that a
	vertex $u \in U_i$ is processed and $d(x_u,x_c)\leq\alpha$ for some $c\in C$.
	The previous computation and the event we conditioned on imply
	that $c\in U_i$. Thus, all $u\in U_i$ are assigned to the same $c\in C$.
\end{proof}

Next, we show that all left-side vertices have degree $O(s)$.
\begin{lemma}
\label{lem:bmf:degree}
	With probability at least $1-n^{-2}$, each vertex $u\in U$ has degree $O(s)$.
\end{lemma}
\begin{proof}
	Let $u \in U_i$ and let $d(u)$ be the degree of $u$. Then we get that
		$\Exp{d(u)} = p |V_i| + q |V \setminus V_i|
			\leq p s + (K_1 s / n) n
			= O(s)$.
	Since $\Exp{d(u)} \geq p |V_i| \geq K_3 p \lg n$,
	we can apply a Chernoff bound to obtain that for large enough $K_3$ it holds
	that
	$d(u) \in [0.99\Exp{d(u)},1.01\Exp{d(u)}]$ with probability at
	least $1-n^{-2}$.
\end{proof}

Now we show that Algorithm~\ref{algo:bmf:greedy} indeed returns the correct
right-side clusters if we set $\theta = 0.75 p$.
\begin{lemma}
\label{lem:bmf:postprocessing}
	With high probability Algorithm~\ref{algo:bmf:greedy} returns clusters
	$\tilde{V_1},\dots,\tilde{V_k}$
	such that $\{ \tilde{V_1},\dots,\tilde{V_k}\} = \{V_1,\dots,V_k\}$.
\end{lemma}
\begin{proof}
	Condition on the events from Lemma~\ref{lem:bmf:correct-clustering}. Let
	$i \in [k]$ and suppose $c \in C$ satisfies $c \in U_i$.  We show
	$\tilde{V_c} = V_i$.

	Consider the heavy hitters data structure $\MG(c)$. Recall that when a
	vertex $u\in U$ is assigned to $c$, we added all $j \in [n]$ to $\MG(c)$
	with $(u,v_j)\in E$. Hence, the stream $X$ of numbers that were
	processed by $\MG(c)$ satisfies that the frequency $f_j$ of $j$ is exactly
	$f_j = |\{ u \in U_i : (u,v_j) \in E \}|$.
	
	From the random graph model we get that $\Exp{f_j} = p |U_i|$ if $v_j \in
	V_i$ and $\Exp{f_j} = q |U_i|$ if $v_j \not\in V_i$.  Using
	a Chernoff bound and $|U_i| \geq K_2 \lg n$, we get that when $K_2$
	is large enough, $f_j > 0.99 p |U_i|$ if $v_j \in V_i$ and
	$f_j < 1.01q |U_i| \leq 0.5 p |U_i|$ if $v_j \not\in V_i$ with probability at least $1-n^{-2}$. Using a
	union bound, we get that the previous event holds for all $j \in [n]$
	simultaneously with probability at least $1-n^{-1}$.  We condition on this
	event for the rest of the proof.

	The total number of points inserted into $\MG(c)$ is $|X| =
	\sum_{u \in U_i} d(u)$ and using Lemma~\ref{lem:bmf:degree} and a union bound,
	$|X| = O(|U_i| s)$ with high probability. Thus, if we run
	$\MG(c)$ with $\varepsilon = C p / (2s)$ for some suitable constant $C$, we
	get that $\MG(c)$ uses space $O(1/\varepsilon) = O(s)$ and provides an
	approximation $\hat{f_j}$ of each $f_j$ within additive error $\varepsilon
	|X| \leq 0.1 p |U_i|$.

	Thus, if $v_j \in V_i$ then $\hat{f_j} \geq f_j - \varepsilon|X| \geq 0.89 p |U_i|$
	and if $v_j\not\in V_i$ then $\hat{f_j} \leq f_j + \varepsilon|X| \leq 0.6 p |U_i|$.
	Setting $\theta = 0.75 p$ we get that the algorithm satisfies
	$\tilde{V_c} = V_i$.
\end{proof}

Now we analyze the space and running time of the algorithm.
\begin{lemma}
	W.h.p.\ the space usage of Algorithm~\ref{algo:bmf:greedy} is $O(ks)$ and its
	running time is $O(mks)$.
\end{lemma}
\begin{proof}
	Conditioning on Lemma~\ref{lem:bmf:correct-clustering}, the algorithm only
	stores $k$ centers.  Storing a single center takes space $O(s)$ to store its
	neighbors by Lemma~\ref{lem:bmf:degree}. Furthermore, for a single center we
	need to store its heavy hitters data structure.  As we argued in the proof
	of Lemma~\ref{lem:bmf:postprocessing} it suffices to use the heavy hitters data
	structure with $O(s)$ counters for each center. Thus, the total space usage
	is $O(ks)$.

	Observe that for each $u\in U$ the running time is dominated by computing $d
	= \min_{c\in C} d(x_u, x_c)$. As there are only $k$ centers $c\in C$ and
	since all $u\in U$ and $c\in C$ have only $O(s)$ neighbors, we can compute
	$d$ in time $O(ks)$. Merging the heavy hitters data structures can be done
	in constant amortized time. Thus, the total running time for the pass over
	the stream is $O(mks)$ since $|U|=m$. In the postprocessing step, we only
	spend time $O(ks)$ because each of the heavy hitters data structures only
	contains $O(s)$ counters.
\end{proof}

\subsection{Proof of Proposition~\ref{prop:bmf:space}}
	Any algorithm to solve the biclustering problem must be able to output the
	planted clusters $V_1,\dots,V_k$. Suppose that each $V_i$ consists of $s$
	vertices and that all $V_i$ are mutually disjoint. Then there are
	$\binom{n}{ks}$ possibilities to pick the $V_i$.  Thus, any algorithm that
	is able to return the $V_i$ exactly must use at least $\lg \binom{n}{ks} =
	\Omega(\lg n^{ks}) = \Omega(ks \lg n)$ bits. Since the standard word RAM
	model of computation is considering words of size $\Theta(\lg n)$, this
	yields a lower bound of $\Omega(ks)$ space.

\subsection{Analysis Sketch for \sofa}
\label{sec:bmf:analysis-importance}
We note that it is possible to adapt the above analysis to obtain provable
guarantees for \sofa (Algorithm~\ref{algo:bmf:importance}) but they are weaker
than the ones from Theorem~\ref{thm:bmf:algo-toy}.  We now provide a rough
sketch of how to obtain these guarantees.

Essentially, we would like to follow the same strategy as in the proof of
Theorem~\ref{thm:bmf:algo-toy}: First, we show that all left-side vertices are
clustered correctly (as in Lemma~\ref{lem:bmf:correct-clustering}). Second, we
show that (if the first step succeeded) the correct clusters are returned (as in
Lemma~\ref{lem:bmf:postprocessing}). While this high-level strategy remains the
same, some adjustments have to be made in the details.

For the first part, we would like to obtain a result similar to
Lemma~\ref{lem:bmf:correct-clustering} stating that the algorithm clusters the
left-side vertices correctly. However, this is not correct for the streaming
$k$-Median algorithm in~\cite{braverman11streaming} which we use as a
subroutine.  However, one can show that if the distance of vertices from
different clusters is a factor $\Omega(k)$ larger than the distance of vertices
from the same cluster, then the algorithm from~\cite{braverman11streaming}
clusters only an $\varepsilon$-fraction of the vertices incorrectly.
More formally, one can show that under some conditions on the $U_i,V_i,p$ and
$q$, there exists a constant $C$ such that if $Ck \cdot d(x_u,x_{u'}) <
d(x_u,x_{u''})$ for $u,u'\in U_i$ and $u''\in U_{j}$, $j\neq i$, then only an
$\varepsilon$-fraction of the vertices in $U$ is clustered incorrectly. Here,
the constant $C$ depends on the approximation ratio of the algorithm
in~\cite{braverman11streaming}. This is a standard result in the clustering
community. Note that the previous condition on the distances of vertices from
different clusters is much stronger than what we used in the proof of
Lemma~\ref{lem:bmf:correct-clustering} (where the distances only differ by a
small constant factor).

Now, with the new version of Lemma~\ref{lem:bmf:correct-clustering}, we can go
through the proof of Lemma~\ref{lem:bmf:correct-clustering} and observe that it
still holds if $p\geq 1-1/O(k)$ (which again is a stronger assumption than in
Theorem~\ref{thm:bmf:algo-toy} where we assumed $p\in[1/2,0.99]$).

Finally, for the second step of our strategy, we observe that the analysis in
Lemma~\ref{lem:bmf:postprocessing} can be adjusted to the setting where a
$\varepsilon$-fraction of the vertices was clustered incorrectly (this only
changes the values of $\Exp{f_j}$ slightly and the overall proof strategy still
works under some conditions on $\varepsilon$ and the cluster size).

\section{Related Work}
\label{sec:bmf:related}
Random graph models for bipartite graphs as presented in
Section~\ref{sec:bmf:preliminaries} are usually studied under the name bipartite
stochastic block models (SBMs)~\cite{abbe2018community}. This problem has received
attention in the past~\cite{xu14jointly,lim15convex} and recently it was shown
that in bipartite graphs even very small clusters can be
recovered~\cite{neumann18bipartite,zhou19analysis,razaee19matched}.
Furthermore, if all clusters have size $\Omega(n)$, algorithms achieving
the information-theoretically optimal recovery thresholds were
presented~\cite{abbe16community,abbe16recovering,zhou18optimal}.
However, these algorithms do not work in the
streaming setting and (on the hardware we used) none of them
would be able to process the real-world datasets we considered in
Section~\ref{sec:bmf:experiments}.

Yun et al.~\cite{yun14streaming} studied SBMs in a streaming setting and
provided algorithms using $O(n^{2/3})$ bits of space when the clustering does
not have to be stored explicitly. However, their algorithm does not apply to
bipartite graphs and it assumes that all clusters have size $\Omega(n)$ which is
unrealistic in bipartite graphs as we discussed in the introduction.

Alistarh et al.~\cite{alistarh15streaming} consider a biclustering problem in
random graphs which is similar to the one studied in this paper. They provide
guarantees for recovering the left-side clusters of the graph, but they do not
provide recovery guarantees for the right-side clusters.  Furthermore, their
data generating model is more simplistic than the one used in this paper and
their algorithm can require up to $O(kn)$ space in practice.

The BMF problem
was introduced in the data mining community by Miettinen et
al.~\cite{miettinen08discrete}
and has been popular in this community ever since
\cite{miettinen12dynamic,lucchese14unifying,miettinen14mdl4bmf,hess18trustworthy,rajca19parallelization,osicka17boolean}.
Recently, the problem was also studied in the machine learning
community~\cite{ravanbakhsh16boolean,rukat17bayesian,rukat18tensormachine,liang19noisy,kumar19faster}
and the theory community~\cite{ban19ptas,fomin18approximation}. 
The only streaming algorithm for BMF is by
Bhattacharya et al.~\cite{bhattacharya19streaming}, who provided a
$4$-pass streaming algorithm which computes a $(1+\varepsilon)$-approximate
solution for BMF.  However, their algorithm is of rather theoretical nature
since it requires space
$O(n \cdot (\lg m)^{2k} \cdot 2^{\tilde{O}(2^{2k}/\varepsilon^2)})$ and since it
uses exhaustive enumeration steps which are slow in practice.
Chandran et al.~\cite{chandran16parameterized} showed that under a standard
assumption in complexity theory, any approximation algorithm for
BMF requires time $2^{2^{\Omega(k)}}$ or
$(mn)^{\omega(1)}$; this essentially rules out practical algorithms for BMF with
approximation guarantees.
For a recent survey on BMF see~\cite{miettinen2020recent}.

We are not aware of any algorithm which (like \sofa) performs a single pass over
the left-side vertices of a bipartite graph and then returns a clustering of the
right-side vertices.

\section{Conclusion}
\label{sec:bmf:conclusion}

We presented \sofa, the first algorithm which after single pass over the
left-side vertices of a bipartite graph returns the right-side clusters using
sublinear memory. We showed that after a second pass over the stream, \sofa
solves biclustering and BMF problems.  Our experiments showed that \sofa is
orders of magnitude faster and more memory-efficient than a static baseline
algorithm while still providing high quality results.  Furthermore, we proved
that under a standard random graph model, a version of \sofa can find the
planted clusters under a natural separation condition.  In future work it will
be interesting to consider streaming settings in which the edges arrive one by
one.  Since the main building blocks of \sofa (coresets and mergeable heavy
hitters data structures) extend to distributed settings, it will be interesting
to make \sofa distributed.

\section*{Acknowledgments}

We are deeply grateful to Vincent Cohen-Addad for helpful discussions during
early stages of this project and for telling the analysis in
Section~\ref{sec:bmf:analysis-importance}.  SN gratefully acknowledges the
financial support from the Doctoral Programme ``Vienna Graduate School on
Computational Optimization'' which is funded by the Austrian Science Fund (FWF,
project no.\ W1260-N35) and from the European Research Council under the
European Community's Seventh Framework Programme (FP7/2007-2013) / ERC grant
agreement No.~340506.

\bibliographystyle{plain}
\bibliography{main}

\end{document}